\documentclass{article}

\usepackage[preprint]{neurips_2026}

\usepackage{amsmath, amssymb, amsthm}
\usepackage{booktabs}
\usepackage[colorlinks=true, linkcolor=blue, citecolor=blue, urlcolor=blue, breaklinks=true, pdfborder={0 0 0}]{hyperref}
\usepackage{graphicx}
\usepackage{xcolor}
\usepackage{algorithm}
\usepackage{algpseudocode}
\usepackage{natbib}
\usepackage{microtype}
\usepackage{wrapfig}
\usepackage{enumitem}
\usepackage{tabularx}
\usepackage{makecell}
\usepackage[capitalize,noabbrev]{cleveref}

\definecolor{edgeblue}{HTML}{1F5AB5}

\newcommand{\nbias}{\mathrm{ENA}}

\newtheorem{lemma}{Lemma}

\newtheorem{remark}{Remark}

\title{When (and How) to Trust the Expert: Diagnosing Query-Time Expert-Guided Reinforcement Learning}

\author{
  Yann Berthelot \\
  UMR 9189 -- CRIStAL \\
  Universit\'e de Lille \\
  CNRS, Inria, Centrale Lille \\
  Lille, France \\
  Saint-Gobain Research Paris \\
  \texttt{yannberthelot1@gmail.com}
  \And
  Philippe Preux \\
  UMR 9189 -- CRIStAL \\
  Universit\'e de Lille \\
  CNRS, Inria, Centrale Lille \\
  Lille, France \\
  \texttt{philippe.preux@inria.fr}
  \And
  Riad Akrour \\
  UMR 9189 -- CRIStAL \\
  Universit\'e de Lille \\
  Inria, CNRS, Centrale Lille \\
  Lille, France \\
  \texttt{riad.akrour@inria.fr}
}

\begin{document}
\maketitle

\begin{abstract}

Many continuous-control problems ship with a competent but
suboptimal controller (a tuned PID, a hand-designed gait). A
growing family of methods uses such controllers as queryable
experts during RL, but each method has been proposed in isolation,
on a different benchmark, without imperfect-expert testing. We
harmonize the comparison on a shared SAC backbone, common HPO and
evaluation protocols, $100/50$ seeds per (env, method), and a
degradation sweep over expert undertuning, action bias, and
observation noise. The comparison surfaces three failure modes
single-paper evaluations miss: \textbf{(F1)} a critic blind spot
under argmax-plus-bootstrap that drags IBRL below no-expert SAC on
experts close to the no-expert-RL ceiling (\emph{RL-near-ceiling},
distinct from the absolute physical ceiling); \textbf{(F2)}
residual saturation on far-from-optimal experts; and \textbf{(F3)}
warm-start buffer poisoning that collapses training-time-handoff
methods under deployment-time expert undertuning. \emph{No single
method dominates}: each wins on one task-structure regime and
fails predictably elsewhere; on RL-near-ceiling experts (FourTank,
GlassFurnace) no query-time method clears the expert within our
$1\mathrm{M}$-step budget, leaving open whether this is a
fundamental wall or a budget effect. We convert the spread into a
testable decision rule keyed on three pre-training observables
(expert quality, task termination, perturbation type). The
benchmark, taxonomy, and decision rule are the primary
contribution; we additionally describe \textbf{EDGE}, a
softmax-over-ensemble-LCB design point used to demonstrate that
both axes the taxonomy points to (gate form, scoring rule) are
individually exploitable.

\end{abstract}

\section{Introduction}
\label{sec:intro}

Industrial control, locomotion, and trajectory-tracking tasks
rarely ship without a working baseline: a hand-engineered feedback
law (PID \citep{astrom2005amigo}, LQR \citep{kalman1960lqr},
MPC \citep{rawlings2017mpc}), a motion planner, a central pattern
generator (CPG \citep{ijspeert2008cpg}), or a behavior-cloned
policy. These experts encode useful knowledge but are rarely
optimal, and they are typically tuned on a nominal model that does
not match the deployed system: the sim-to-real gap and modeling
error make the shipped controller a non-trivial floor that RL is
asked to surpass. The practitioner's question is therefore not
\emph{whether} such an expert can help RL, but \emph{how}: which
integration mechanism extracts the most value from it, and under
what failure modes.

Each query-time-expert method (IBRL \citep{hu2023ibrl}, the JSRL
family \citep{uchendu2023jsrl}, Residual SAC
\citep{johannink2019residual}, offline-demonstration relatives in
\cref{sec:related}) was proposed in isolation, on a different
benchmark, against a different expert. Two gaps follow for a
practitioner choosing among them: no head-to-head comparison on a
common protocol, and sparse coverage of robustness to expert
imperfection. We address both.

\textbf{Contributions.}
\begin{itemize}[leftmargin=1.2em,itemsep=2pt,topsep=2pt]
\item A common-protocol benchmark of five expert-guided methods
(IBRL, two JSRL variants, Residual SAC, EDGE) on four tasks, with a
shared SAC backbone, shared HPO, $100$ seeds per (env, method) on
control tasks and $50$ on locomotion, and a degradation sweep
covering expert undertuning, action bias, and observation noise.
Two of the four tasks (Plane3DCircle, GlassFurnace) are released
with this paper as new continuous-control benchmarks; they target
regimes underrepresented in expert-guided RL evaluations
(terminating crash dynamics with a structurally inadequate PID,
and long-horizon thermal MIMO control respectively).
\item Three mechanism-specific failure modes (F1--F3,
Section~\ref{sec:failuremodes}) predicting which method fails on
which task structure; F1 and F3 only become visible after
harmonization.
\item A task-structural decision rule
(Section~\ref{sec:when-to-use}) mapping termination, expert quality,
and expected deployment perturbation to a method choice, framed as
testable guidance.
\item Secondarily, \textbf{EDGE}: a quality-aware pessimism-gated
mixing design point (Section~\ref{sec:method}) showing that the
two design axes the taxonomy identifies (gate form, scoring rule) are
each individually exploitable. Competitive on every task and
co-best with JSRL-curriculum under expert-undertuning, offered as a design
probe, not a SOTA claim.
\end{itemize}

\section{Background and Notation}
\label{sec:problem}

This work studies expert-guided RL methods and contributes a new one
(EDGE, \cref{sec:method}) inspired by Thompson sampling. We briefly
recap both before the experiments.

\paragraph{Expert-guided RL.}
We consider a discounted continuous-control Markov decision problem (MDP)
$(\mathcal{S}, \mathcal{A}, T, R, \gamma)$ with state space
$\mathcal{S} \subseteq \mathbb{R}^n$, action space
$\mathcal{A} = [-1,1]^d$, transition function $T: \mathcal{S} \times
\mathcal{A} \to \mathcal{S}$, and reward $R: \mathcal{S} \times
\mathcal{A} \to \mathbb{R}$. Alongside the MDP we are given a
deterministic \emph{expert controller} $\pi^{\mathrm{exp}} : \mathcal{S} \to
\mathcal{A}$, typically a hand-engineered PID, LQR, MPC, or pattern
generator, that can be queried at any state visited during training.
The agent's objective is the standard $J(\pi) = \mathbb{E}\left[
\sum_{t=0}^\infty \gamma^t R(s_t, a_t)\right]$. We denote the
expert's return as $J_{\mathrm{exp}}$, the no-expert SAC baseline's
return at convergence as $J_{\mathrm{RL}}$ (the best non-expert
floor on each task in our benchmark, used as the denominator of
expert-quality below), and the learner as $\pi_\theta$.

\paragraph{Thompson sampling.}
For a $K$-armed bandit with posterior $p_t(\theta)$ over reward
parameters at step $t$, Thompson sampling (TS)
\citep{russo2014ts,russo2018tutorial} draws $\tilde\theta_t \sim
p_t$ and plays $a_t = \arg\max_{a \in [K]} Q(a;\tilde\theta_t)$.
Pessimistic variants \citep{wu2016conservative,kazerouni2017conservative}
replace the sample by a lower-confidence-bound (LCB) surrogate
$\check Q(a; p_t) = \mu(a; p_t) - \kappa\,\mathrm{spread}(a; p_t)$,
where $\mu(a; p_t)$ is the posterior mean reward of arm $a$ under
$p_t$ and $\mathrm{spread}(a; p_t)$ is a posterior dispersion
functional (e.g.\ standard deviation, IQR, or ensemble range), and
the confidence-scaling coefficient $\kappa \geq 0$ controls how
aggressively low-uncertainty arms are preferred.

\section{Related Work}
\label{sec:related}

\paragraph{Benchmarking expert-guided RL.}
No existing benchmark provides a head-to-head common-protocol
comparison of query-time-expert RL methods. Each method we compare
was introduced in a separate paper, with self-chosen baselines,
tasks, expert qualities, and HPO budgets. Adjacent benchmarks
target different settings: D4RL \citep{fu2020d4rl} and Robomimic
\citep{mandlekar2021robomimic} standardize \emph{offline} RL on
fixed demonstration datasets; rliable \citep{agarwal2021rliable}
supplies aggregation tools (which we adopt) but is not itself a
benchmark of expert-guided methods.

\paragraph{Expert-leveraging methods.}
We compare three families of query-time expert use: argmax-action
selection (IBRL \citep{hu2023ibrl}; the recent DRLR
\citep{shen2026drlr} adds a critic-side regularizer on top of
IBRL, orthogonal to the gate-form vs scoring-rule axes of our
taxonomy and evaluated on a different task suite, so we do not
re-run it head-to-head), in-episode handoff with a decaying
horizon (JSRL family \citep{uchendu2023jsrl}, curriculum and
warm-start variants), and additive correction (Residual SAC
\citep{johannink2019residual}).
Adjacent families consuming a fixed-dataset rather than a callable
expert (DAgger \citep{ross2011dagger}, AWAC \citep{nair2020awac},
IQL \citep{kostrikov2022iql}, PEX \citep{zhang2023pex}, LOGO
\citep{rengarajan2022logo}) need different fairness conventions and
are out of scope.

\paragraph{Two related-but-not-included neighbors.}
A \emph{Parameterized-Expert-Policy} baseline (RL actor modulates
the controller's parameters rather than mixing actions) is
structurally distinct: its expressiveness is bounded by the
controller's parametric form (a CPG, even with per-step
modulation, cannot express arbitrary actuator patterns), so any
comparison conflates the integration mechanism with a different
policy class. Seeding such a parameterization from our experts is
a useful follow-up. \emph{RLPD} \citep{ball2023rlpd} is the
natural static-buffer comparison; we exclude it as a scoping
artifact (post-dated our baseline freeze), not for a principled
reason.

\section{Expert-Driven Guided Exploration (EDGE)}
\label{sec:method}

\paragraph{Motivation: critic blind spot under argmax.}
IBRL-style argmax selection (\cref{tab:baselines}) between expert
and policy actions writes expert transitions to the replay buffer
whenever the expert wins in $Q$, anchoring the critic on the expert
action but not on the policy's own. The SAC
actor reads $Q_\phi(s, \pi_\theta(s))$, so it optimizes against a
target that stays unanchored on the expert-dominant set, which we
call the \emph{critic blind spot}. EDGE replaces argmax with a
state-dependent stochastic gate, so every visited state keeps
receiving fresh policy-action Bellman targets at positive rate
$(1-p_t(s))\,\mu_t(s)$, where $\mu_t(s)$ is the on-policy state-visit
measure under the behavior policy at step $t$, and no persistent
blind-spot region can form.
Formal Remark + Coverage Lemma in \cref{app:proofs}.

\subsection{Algorithm}

EDGE is a minimal modification to any off-policy actor-critic
algorithm: it changes only the action-selection rule and the input
to the actor and critic, leaving the gradient updates untouched.
\begin{wrapfigure}[18]{r}{0.66\textwidth}
\vspace{-2em}
\begin{minipage}{0.66\textwidth}
\begin{algorithm}[H]
\caption{EDGE. {\color{edgeblue} Blue} = modifications to base SAC.}
\label{alg:edge}
\begin{algorithmic}[1]
\State \textbf{Inputs:} steps $T$, ensemble size $N$, discount $\gamma$;
{\color{edgeblue} pessimism $\kappa\!\geq\!0$, gate temperature $\tau\!>\!0$; expert $\pi^{\mathrm{exp}}$ with internal state $z$.}
\State Init $\pi_\theta$, critics $\{Q_{\phi_n}\}_{n=1}^N$, targets $\{\bar Q_{\bar\phi_n}\}_{n=1}^N$, buffer $\mathcal{B}$.
\State Reset env to $s_0$; {\color{edgeblue} init expert state $z_0$.}
\State {\color{edgeblue} $\tilde s_0 \leftarrow [s_0, z_0]$}
\For{$t = 0, \ldots, T-1$}
    \State {\color{edgeblue} $\check{Q}(\tilde s_t, a) \!=\! \min_n Q_{\phi_n}(\tilde s_t, a) \!-\! \kappa(\max_n Q_{\phi_n} \!-\! \min_n Q_{\phi_n})$}
    \State {\color{edgeblue} $p_t = \sigma\big((\check{Q}(\tilde s_t, \pi^{\mathrm{exp}}(s_t, z_t)) - \check{Q}(\tilde s_t, \pi_\theta(\tilde s_t)))/\tau\big)$}
    \State {\color{edgeblue} $b_t \sim \mathrm{Bernoulli}(p_t)$}
    \State $a_t \leftarrow$ {\color{edgeblue}$b_t\,\pi^{\mathrm{exp}}(s_t, z_t) + (1-b_t)\,$}$\pi_\theta(\tilde s_t)$
    \State Execute $a_t$; observe $r_t, s_{t+1}$; {\color{edgeblue}update $z_{t+1}$.}
    \State {\color{edgeblue} $\tilde s_{t+1} \leftarrow [s_{t+1}, z_{t+1}]$}
    \State Store $(\tilde s_t, a_t, r_t,$ {\color{edgeblue}$\tilde s_{t+1}$}$)$ in $\mathcal{B}$.
    \State Update $\pi_\theta$, $\{\phi_n\}_{n=1}^N$, $\{\bar\phi_n\}_{n=1}^N$ via base SAC.
\EndFor
\end{algorithmic}
\end{algorithm}
\end{minipage}
\end{wrapfigure}
Two pieces compose EDGE: (a) a quality-aware pessimism gate
that selects between expert and policy at every step, and (b) an
expert-state augmentation $[s, z]$ exposing the expert's internal
controller state (PID integrator state, CPG phase) to actor and
critic. The augmentation is applied uniformly to every
expert-using method (IBRL, both JSRL variants, Residual SAC,
EDGE), as a property of the comparison protocol, not a private
EDGE knob; the \texttt{no\_state\_aug} ablation
(\cref{app:ablation}) measures its contribution at fixed protocol.
Unlike DAgger \citep{ross2011dagger} which labels visited states
with the expert's action and trains by behavior cloning against those
labels, our replay tuple records the executed action with its realized
reward and next state, never the expert's label, so the learner is
not capped at the expert's behavior and can eventually exceed it.

EDGE wraps any off-policy actor-critic algorithm by changing only the
behavior policy and the actor/critic input; gradient updates remain
those of the base learner. The wrapper introduces three modifications,
marked blue in \cref{alg:edge}: pessimism-gated stochastic mixing
(vs.\ IBRL argmax, testing \cref{rem:blindspot}),
expert-state-augmented input $[s, z]$, and two new hyperparameters
$\kappa, \tau$ tuned per task by HPO (\cref{app:hpo}). Each axis is
ablated in \cref{app:ablation}. SAC is the only base learner we
instantiate, so ``EDGE'' refers to ``EDGE on top of SAC'' throughout
the paper. Concretely, EDGE acts under the expert-mixed behavior
policy
\begin{equation}
    \beta_t(a\mid s) \;=\; p_t(s)\,\delta_{\pi^{\mathrm{exp}}(s)}(a)
    \;+\;(1-p_t(s))\,\pi_\theta(a\mid s),
    \quad
    p_t(s) \;=\; \sigma\!\Big(\tfrac{\check{Q}_t(s,\pi^{\mathrm{exp}}(s)) -
                              \check{Q}_t(s,\pi_\theta(s))}{\tau}\Big),
    \label{eq:behavior}
\end{equation}
where $\sigma(x) = (1 + e^{-x})^{-1}$ is the logistic
(sigmoid) function, $\delta_{\pi^{\mathrm{exp}}(s)}(\cdot)$ is the point mass on
the expert action $\pi^{\mathrm{exp}}(s)$ (the expert is assumed deterministic
per \S2; extension to stochastic experts is straightforward, see
\cref{app:limitations_extra}), and
$\check{Q}_t(s,a) = Q^{\min}_\phi(s,a) - \kappa\,(Q^{\max}_\phi(s,a)
- Q^{\min}_\phi(s,a))$ is the ensemble-pessimistic Q-estimate (the
critic-ensemble minimum, further penalized by ensemble spread,
adapting the disagreement signal of offline-RL value penalties
\citep{an2021edac,ghasemipour2022msg} and ensemble exploration
\citep{osband2016bootstrappeddqn,lee2021sunrise} to an
expert-vs-policy mixing decision). The replay buffer stores
$(s, a, r, s')$
tuples with $a \sim \beta_t$; $0 < p_t(s) < 1$ at every state with
non-degenerate critic disagreement, so the gate preserves full
support over both arms. Equivalently, the gate is a fixed-bandwidth
Gaussian approximation to pessimistic Thompson sampling on the
2-armed $\{\pi^{\mathrm{exp}}(s),\pi_\theta(s)\}$ bandit: $\check Q$ replaces the
posterior draw, $\tau$ replaces the sample variance, and the Gaussian
approximation compensates for the under-dispersion of small critic
ensembles \citep{osband2016bootstrappeddqn,an2021edac,ghasemipour2022msg}
that would otherwise collapse a literal sample-based TS to greedy.
The \texttt{literal\_thompson} ablation (\cref{app:ablation})
replaces the softmax-LCB gate with literal Gaussian-parametric
Thompson sampling on the same critic ensemble: on Plane3DCircle
EDGE reaches IQM $\approx 6{,}230$ while literal TS plateaus at
$\approx 4{,}265$ at $300\mathrm{k}$ steps, $50$ seeds. Scaling the
critic ensemble to $N{=}10$ leaves the gap essentially unchanged
($4{,}190$), ruling out under-dispersion as the operative
mechanism: the load-bearing piece is the pessimism shift
($-\kappa\sigma$) shared by both arms in LCB scoring, not the
softmax-vs-Thompson choice. The four-corner sweep over gate form
and scoring rule reported in \cref{app:ablation} corroborates this:
swapping min-$Q$ for LCB lifts both gates substantially while
swapping argmax for softmax adds a smaller increment on top.\label{sec:bandit}

\section{Experiments}
\label{sec:experiments}

The experiments answer four concrete questions about query-time-expert
RL methods on a common protocol.
\textbf{(Q1, common-protocol comparison)} On a shared backbone, HPO
budget, expert, and seed protocol, which expert-guided method wins on
which task, and by how much over the no-expert SAC and expert-only
floors? \textbf{(Q2, failure modes)} Do the predicted failure modes
F1, F2, F3 (defined in \cref{sec:failuremodes}) actually bind on the
regimes the taxonomy targets, and are they distinguishable from
generic underperformance? \textbf{(Q3,
robustness to expert imperfection)} How does each method degrade
under expert undertuning, action bias, and observation noise; do the
clean-protocol rankings reverse? \textbf{(Q4, EDGE design choices)}
Are the two EDGE modifications (pessimism-gated mixing, expert-state
augmentation) each necessary, and how sensitive is performance to the
two added hyperparameters? Q1--Q4 are answered in
\cref{sec:main-results,sec:failuremodes,sec:sensitivity,sec:design-ablation};
the per-axis ablation table backing Q4 is in \cref{app:ablation}.
Throughout this section, the primary metric is the expert-normalized
advantage $\mathrm{ENA} = (J - J_{\mathrm{exp}})/(J_{\mathrm{ref}} -
J_{\mathrm{exp}})$ (formal definition and per-task $J_{\mathrm{ref}}$
in \cref{sec:protocol}, \cref{eq:nbias}; sign convention:
$\mathrm{ENA}\!=\!0$ matches the expert, $\mathrm{ENA}\!=\!1$ matches
a per-task physical ceiling). Code, per-(env, method) hyperparameters, and raw
result JSONs are released as three GitHub repositories covering the task
suite, the RL framework with all five expert-guided methods, and the
experiment harness (\cref{app:reproducibility}).

\subsection{Setup}

\subsubsection{Benchmark Tasks}

Four tasks were chosen to span the regions of expert-quality and
task-structure space that the failure-mode taxonomy
(\cref{sec:failuremodes}) predicts to bind:

\begin{itemize}[leftmargin=1.3em,itemsep=1pt,topsep=2pt]
\item \textbf{FourTank} (PID-controlled MIMO setpoint tracking,
near-MPC-tuned expert): near-ceiling expert, non-terminating; F1-prone.
\item \textbf{Plane3DCircle (new)} (PID-controlled 3D aircraft
trajectory tracking, weak expert): terminating with catastrophic
failure; F2-prone.
\item \textbf{GlassFurnace (new)} (PID-controlled thermal process,
mid-quality expert): long thermal horizon, non-terminating;
F3-prone.
\item \textbf{CheetahRun}: non-terminating CPG-controlled locomotion
with a weak open-loop expert.
\end{itemize}

Each task ships with a reproducibly tuned expert; tuning protocols
and per-env dynamics in \cref{app:tasks,app:tuning}.

\subsubsection{Baselines}
\label{sec:baselines}

All methods share a common hyperparameter optimization (HPO)
protocol (\cref{app:hpo}) and the same base off-policy learner,
Soft Actor-Critic (SAC) \citep{haarnoja2018sac}, so the only varied
axis is the expert-integration mechanism. All methods feed their
actor and critic networks a running-mean/std normalized observation
(implementation detail required for numerical stability on
large-magnitude observations; full discussion in
\cref{app:hpo}). \cref{tab:baselines} lists the six methods compared
in the main body; the equal Optuna \citep{akiba2019optuna} trial
budget across methods is discussed as a methodological caveat in
\cref{sec:limitations}.

\begin{table}[h]
\centering
\small
\renewcommand{\arraystretch}{1.18}
\resizebox{\textwidth}{!}{
\begin{tabular}{lll}
\toprule
\textbf{Method} & \textbf{Mechanism} & \textbf{Tested axis vs. EDGE} \\
\midrule
SAC                  & off-policy actor-critic, no expert       & expert vs.\ no expert \\
Expert               & tuned PID / CPG, no RL                   & RL vs.\ controller-only \\
IBRL                 & argmax over $\{Q(s,\pi^{\mathrm{exp}}),Q(s,\pi_\theta)\}$ + expert bootstrap & argmax vs.\ stochastic gate \\
JSRL-curriculum      & \makecell[l]{per-episode handoff at in-episode step $H_t$,\\$H_t$ decayed from episode horizon to $0$ over training} & hard handoff vs.\ stochastic gate \\
JSRL-training-time (\emph{JSRL-tt}) & expert for first $\rho_{\mathrm{warm}} T$ steps, then policy & curriculum vs.\ warm-start \\
Residual SAC         & $a = \mathrm{clip}(a^e + a_{\mathrm{res}}, -1, 1)$       & additive vs.\ mixed action \\
\midrule
\textbf{EDGE (ours)}                    & pessimism-gated stochastic mixing \eqref{eq:behavior}   & --- \\
\bottomrule
\end{tabular}
}
\caption{Baselines compared in the main body. Each row lists the
method, the mechanism by which it consumes the expert at every visited
state, and the design axis it isolates against EDGE. All expert-using
methods share the same SAC backbone (\cref{sec:baselines}); the only
varied axis across rows is the expert-integration mechanism. IBRL's
expert-action TD bootstrap is ablated separately in
\cref{app:ablation}. The \emph{Expert} row is the controller-only
floor (PID for setpoint tasks, CPG for locomotion).}
\label{tab:baselines}
\end{table}

\subsubsection{Evaluation Protocol}
\label{sec:protocol}

Final phase uses $100$ seeds per configuration on the three control
tasks and $50$ seeds on CheetahRun.\footnote{Lower CheetahRun seed
count: parallel-seed throughput is GPU-RAM-bound on the locomotion, due to observation size in the replay buffer.
task.} Training is $1\mathrm{M}$ environment steps. Per-seed scalars
are the mean episodic return over the last $200\mathrm{k}$ steps;
across-seed aggregation is the IQM with $95\%$ stratified bootstrap
CIs.

\paragraph{Expert-Normalized Advantage (ENA).} The single primary
metric used throughout the paper (figures, headline tables) is the
expert-normalized advantage
\begin{equation}
\nbias \;=\; \frac{J - J_{\mathrm{exp}}}{J_{\mathrm{ref}} - J_{\mathrm{exp}}},
\label{eq:nbias}
\end{equation}
where $J_{\mathrm{exp}}$ is the expert controller's mean return and
$J_{\mathrm{ref}}$ is a fixed task-specific reference scale equal to
the per-task episodic-return ceiling (per-step reward upper bound
times episode length): $10{,}000$ on Plane3DCircle, $500$ on FourTank,
$5{,}760$ on GlassFurnace,
$1{,}000$ on CheetahRun. The ceilings are physical upper bounds, not
empirical scores, and do not depend on any method evaluated in this
paper. Sign convention: $\mathrm{ENA}=0$ matches the expert,
$\mathrm{ENA}=1$ matches the reference scale, $\mathrm{ENA}>1$ exceeds
it, $\mathrm{ENA}<0$ underperforms the expert. ENA is invariant to
per-task reward translation/scaling and lets us aggregate across tasks
with very different absolute returns.

\paragraph{Evaluation criterion.} A candidate method must
\emph{simultaneously} exceed both $J_{\mathrm{exp}}$ and $J_{\mathrm{RL}}$
on every task where either baseline leaves measurable headroom; the
expert is a floor (not a target) and the no-expert SAC baseline
guards against methods whose advantage comes from extra training
budget rather than expert use.

Statistical analysis has two layers: (1) an exploratory
Mann-Whitney $U$ battery vs.\ no-expert SAC with Holm-Bonferroni
correction (\cref{tab:main_perenv}); (2) three directional one-sided
confirmations of F1/F2/F3, with the sign fixed in advance by the
mechanism. Full protocols and the F3 caveat are in
\cref{app:significance,app:directional_tests}.

\subsection{Main Results}
\label{sec:main-results}

\textbf{No single expert-guided method dominates across task structures.}
The benchmark's value is in surfacing this spread under a common
protocol; per-mechanism analysis follows in \cref{sec:failuremodes}.
Robustness is in \cref{sec:sensitivity}; aggregation uses
rliable-style stratified bootstrap \citep{agarwal2021rliable}.

\paragraph{Asymptotic performance ($1\mathrm{M}$ steps).}
Residual SAC ties for the top on FourTank (near-ceiling expert, F1
territory) but collapses on Plane3DCircle (weak expert, F2 territory);
IBRL is dominated by no-expert SAC on FourTank (one-sided permutation,
$p{=}0.031$); JSRL-training-time tops Plane3DCircle clean but is the
most undertuning-fragile method. On FourTank and GlassFurnace
(RL-near-ceiling experts; expert at $98\%$ and $80\%$ of
$J_{\mathrm{RL}}$ respectively) no expert-using method beats the
no-expert SAC baseline by more than $\sim$$1\%$ \emph{within our
$1\mathrm{M}$-step budget}: the contribution from these two tasks
is \emph{negative}, delineating a regime where the entire
query-time-expert family does not yet help. We cannot distinguish
a fundamental wall from a training-budget effect at this horizon;
longer-horizon runs are a natural follow-up. \cref{tab:main_perenv}
reports raw per-env IQM with significance vs no-expert SAC;
\cref{fig:headline} reports the same data on the unified ENA scale.

\paragraph{Sample efficiency.}
On near-ceiling experts, \textbf{Residual SAC inherits the expert at
no learning cost}: a property of the bound on $a_{\mathrm{res}}$,
not a learned outcome. The residual correction collapses to near
zero on a near-optimal expert, so the deployed action matches the
expert from step zero; on FourTank Residual is at expert-parity
already at the first evaluation step while other methods need
hundreds of thousands of steps. The same bound saturates Residual
on weak experts (F2, \cref{sec:failuremodes}): inheritance and
saturation are two halves of one bound. Sample-efficiency rankings
track the asymptote ranking on the other tasks; per-method
first-permanent-crossing curves are in \cref{fig:sample_eff_appendix}.

\begin{table}[h]
\centering
\small
\renewcommand{\arraystretch}{1.15}
\begin{tabular}{lllll}
\toprule
\textbf{Method} & \textbf{P3D-Circle} & \textbf{FourTank} & \textbf{GlassFurnace} & \textbf{CheetahRun} \\
\midrule
SAC (baseline) & $-$6{,}099 & \phantom{+}243.1 & \phantom{+}5{,}406 & \phantom{+}666.9 \\
IBRL & \phantom{+}9{,}536$^{***}$ & \phantom{+}239.4$^{*}$ & \phantom{+}5{,}413$^{**}$ & \phantom{+}813.5$^{***}$ \\
JSRL (curriculum) & \phantom{+}9{,}291$^{***}$ & \phantom{+}239.0$^{*}$ & \phantom{+}5{,}405 & \phantom{+}817.5$^{***}$ \\
JSRL (training-time) & \phantom{+}9{,}651$^{***}$ & \phantom{+}237.2$^{***}$ & \phantom{+}5{,}409 & \phantom{+}773.5$^{***}$ \\
Residual SAC & \phantom{+}3{,}211$^{***}$ & \phantom{+}250.4$^{***}$ & \phantom{+}5{,}424$^{***}$ & \phantom{+}738.8$^{***}$ \\
\midrule
\textbf{EDGE (ours)} & \phantom{+}9{,}409$^{***}$ & \phantom{+}242.7 & \phantom{+}5{,}416$^{***}$ & \phantom{+}751.2$^{***}$ \\
\bottomrule
\end{tabular}

\caption{Per-env asymptotic IQM final return at $1\mathrm{M}$
training steps ($100$ seeds for control, $50$ for CheetahRun). All
cells are raw IQM in the task's native units; SAC (no expert) is
the anchor row. Significance markers on each non-SAC row are
two-sided Mann-Whitney $U$ tests of \emph{that method vs.\ SAC} on
the per-seed final-window scalars, Holm-Bonferroni-corrected
within env over the family of five expert-using methods
($^{*}{:}p_{\mathrm{corr}}{<}0.05$,
$^{**}{:}p_{\mathrm{corr}}{<}0.01$,
$^{***}{:}p_{\mathrm{corr}}{<}0.001$).
Methodology, per-method CIs, and ENA-normalized companion:
\cref{app:perenv,app:significance,fig:headline}.}
\label{tab:main_perenv}
\end{table}

\begin{figure}[h]
\centering
\includegraphics[width=\linewidth]{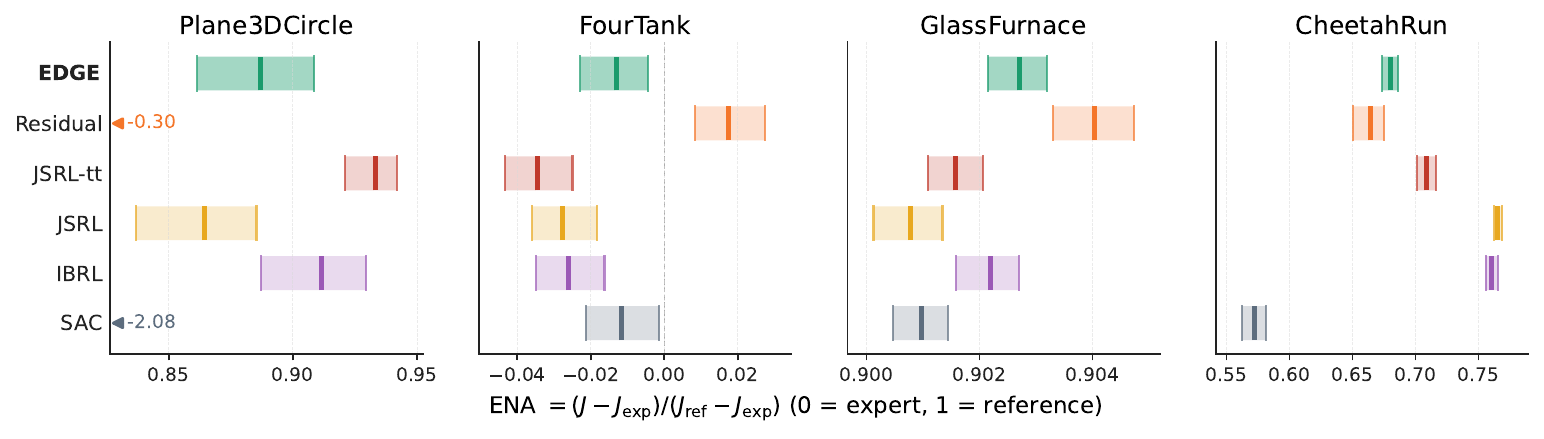}
\caption{Per-env asymptotic Expert-Normalized Advantage at
$1\mathrm{M}$ steps. ENA $= (J - J_{\mathrm{exp}})/(J_{\mathrm{ref}}
- J_{\mathrm{exp}})$, $0$ at the expert, $1$ at the per-task ceiling
(\cref{eq:nbias}). Tick: per-method IQM; rectangle: $95\%$
stratified bootstrap CI over trailing-$200\mathrm{k}$-step per-seed
returns. The Plane3DCircle panel auto-zooms to the cluster of
high-performers; SAC and Residual fall outside and appear as
labelled arrows (full-axis version: \cref{fig:headline_uncensored}).
Same data as \cref{tab:main_perenv}, normalized.}
\label{fig:headline}
\end{figure}

\subsection{Robustness to Expert Sub-Optimality}
\label{sec:sensitivity}

We sweep three perturbation types (intensities per type and method
in \cref{fig:degradation_body}): training-time \emph{expert
undertuning} (log-scale perturbation of the controller's calibrated
parameters), deployment-time \emph{expert action bias} (per-seed
constant offset on the expert's executed action), and deployment-time
\emph{observation noise} (per-step Gaussian noise on the agent's
observation). End-of-training final return is reported against the
original (unperturbed) expert. The annotated companion with per-bar
percent change is \cref{fig:degradation_appendix}.
Reading \cref{fig:degradation_body}: on Plane3DCircle, Residual SAC
is the only method that retains positive return under deployment-time
perturbations despite the lowest clean ceiling, an
asymptote-vs-robustness tradeoff invisible to the undertuning sweep
alone.

\begin{figure}[h]
\centering
\IfFileExists{figs/degradation_v2_bars_by_method_grid.pdf}{
  \includegraphics[width=0.92\linewidth]{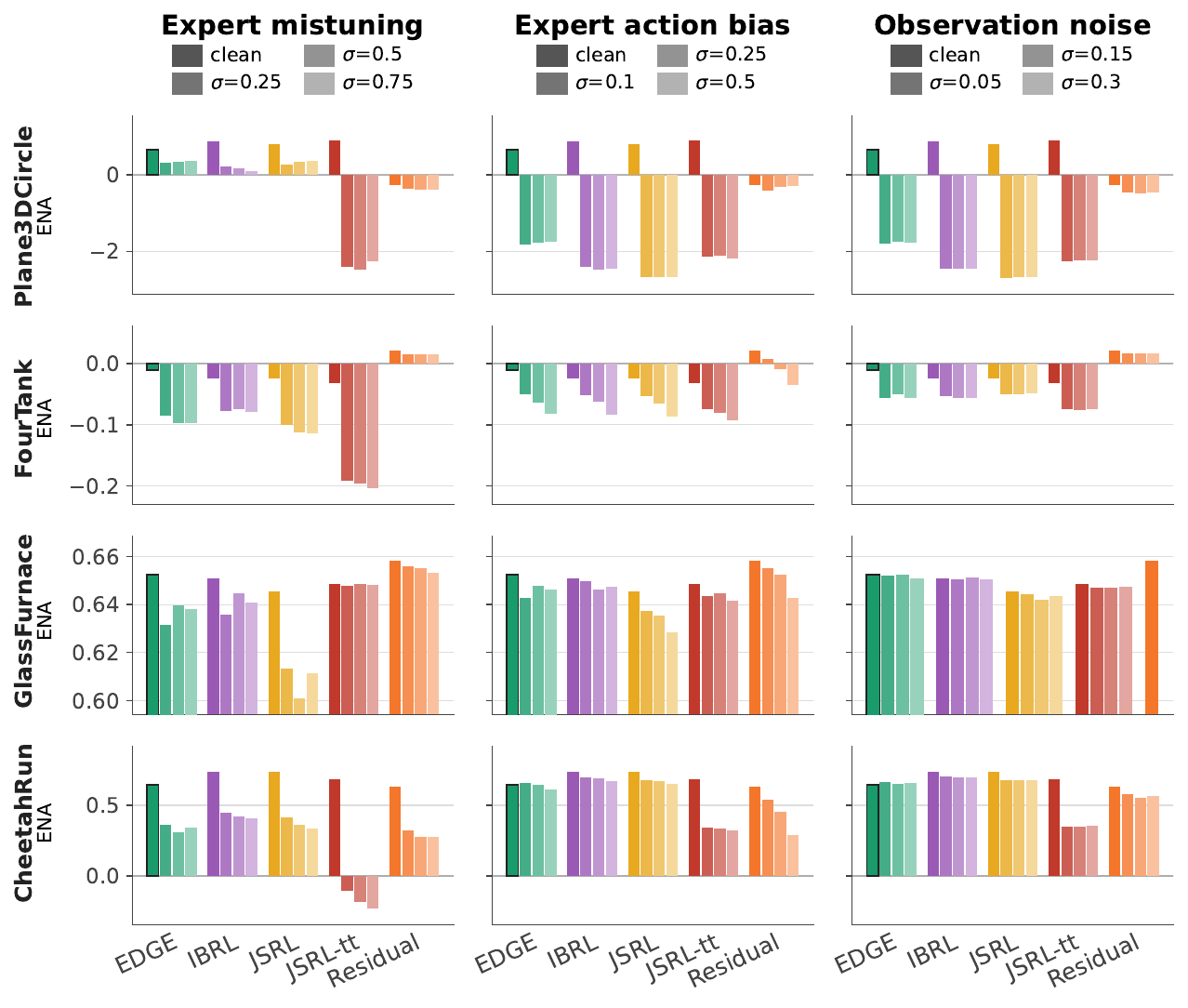}}{}
\caption{Per-environment perturbation sweep ($4\!\times\!3$ grid):
rows are environments, columns are perturbation types
(training-time \emph{expert undertuning}; deployment-time
\emph{expert action bias} and \emph{observation noise}). Each panel
groups the five expert-using methods plus no-expert SAC; for every
method, bars left-to-right are clean (darkest) and increasing $\sigma$.
Bar height is ENA at the perturbed configuration ($50$ seeds);
each row has its own $y$-axis. The annotated companion with per-bar
percent-change labels is \cref{fig:degradation_appendix}; raw-units
companion in \cref{fig:degradation_appendix_raw}.}
\label{fig:degradation_body}
\end{figure}

\subsection{EDGE Design Choices}
\label{sec:design-ablation}

We ablate $13$ EDGE variants per env on Plane3DCircle, CheetahRun,
and FourTank ($50$ seeds, $300\mathrm{k}$ steps; the shorter
horizon keeps all $39$ cells in budget. Full table and per-variant
axis map in \cref{app:ablation}). Both quality-aware modifications
are individually load-bearing on the discriminating env: removing
the softmax gate or the expert-state augmentation each costs
$\sim$$70\%$ of EDGE's return on Plane3DCircle, and an additional
LCB-gated TD bootstrap is approximately neutral. On CheetahRun and
FourTank the variant spread is small (every variant within seed
noise of EDGE except the two sanity-check breakers,
\texttt{random\_expert} and \texttt{store\_policy\_action}),
consistent with both tasks' small between-method spread on the
main results.

\section{Discussion}
\label{sec:discussion}

\subsection{A Taxonomy of Expert-Guided RL Failure Modes}
\label{sec:failuremodes}

The benchmark exposes three mechanism-specific failure modes of
expert-guided RL, each tied to a structural property of the expert
or the integration mechanism. F2 (residual saturation) is a
documented failure mode of additive-correction methods on weak
experts, consistent with prior reports; we keep it in the taxonomy
because the decision rule needs all three labels, but the novelty
here is in F1 and F3, both
of which are surfaced only by the harmonized comparison (F1: IBRL
$<$ no-expert SAC on a near-ceiling-expert task) and the degradation
sweep (F3: training-time-handoff collapse under expert undertuning) and
have not, to our knowledge, been previously named or directionally
tested. We identify each mode post-hoc from the regimes it binds on,
then state the mechanism, the empirical signature, and the
task-structural condition under which the same failure should be
expected to bind on held-out tasks.

\paragraph{F1: Critic blind spot (argmax-gating methods).}
IBRL combines argmax action selection with an expert-action TD
bootstrap (\cref{tab:baselines}). On the states where the expert
beats the policy in $Q$, both ingredients pull the critic toward
the expert action, so the critic never gets fresh on-policy
Bellman targets at the policy action even though the SAC actor
optimizes against it. On near-ceiling experts these states are
persistent, and IBRL falls \emph{below} no-expert SAC. The effect
is real but small. We confirm the predicted sign with a one-sided
permutation test on FourTank ($p{=}0.031$,
\cref{tab:main_perenv,app:directional_tests}); we run it as a single
directional test
keyed on the F1 prediction, not as part of the per-env
Mann-Whitney battery in \cref{tab:main_perenv}, where the
multiple-comparison correction would absorb it. The
\texttt{argmax\_lcb} ablation on FourTank keeps the argmax gate
but drops both the IBRL bootstrap and the min-$Q$ scoring: it
lands between IBRL and SAC, with the bootstrap accounting for
most of the F1 gap and the argmax gate accounting for the rest.
\emph{Anchor and replication.} F1 binds on FourTank
($q^{\mathrm{exp}}\!=\!98\%$ of no-expert SAC return). We also
ran the same one-sided test on GlassFurnace
($q^{\mathrm{exp}}\!=\!80\%$) where IBRL is numerically
\emph{above} no-expert SAC ($\Delta\!=\!+7.9$ IQM, $p{=}0.995$):
F1 does not bind there. Holding $q^{\mathrm{exp}}$ as our
predictor, the F1 binding regime is therefore tighter than
``near-ceiling'' in the loose sense; we narrow the held-out
condition to $q^{\mathrm{exp}}\!\gtrsim\!90\%$.

\paragraph{F2: Residual saturation (additive-correction methods).}
Residual SAC adds a bounded correction to the expert action and
clips back into the action box. The bound cuts both ways. On weak
experts the correction needed to escape the expert's structural
errors exceeds the bound, and Residual gets capped well below the
other methods regardless of training budget; on Plane3DCircle, the
gap to JSRL-tt, IBRL, and EDGE is each significant at
$p{<}10^{-4}$ (\cref{app:directional_tests}). On near-ceiling
experts the same bound flips into an advantage: the correction
collapses toward zero, Residual essentially inherits the expert
from step zero, and the other methods need hundreds of thousands
of steps to catch up. \emph{Anchors:} Plane3DCircle for the
saturation half, FourTank for the free-expert-parity half. F2 is a
tradeoff, not a unilateral failure mode.

\paragraph{F3: Warm-start buffer poisoning (training-time handoff).}
JSRL-training-time first runs the expert for a fixed warm-start
fraction of training, then hands off to the policy. The buffer
collected during warm-start reflects whichever expert was used at
training time, and the policy then has to learn from this buffer
with no further expert input. JSRL-tt has the strongest clean
asymptote in the suite but its return collapses under expert
undertuning. We confirm the predicted direction with a one-sided
permutation test on Plane3DCircle at $\sigma{=}0.5$: JSRL-tt's
per-seed final-window IQM is $-6{,}590$ against EDGE's $-4{,}745$
($\Delta\!=\!-1{,}845$, $p{<}10^{-4}$,
\cref{tab:directional_tests}). On CheetahRun JSRL-tt also loses
most of its clean score under the same perturbation, while no
other expert-using method shows a comparable drop.
\emph{Anchors:} CheetahRun and Plane3DCircle under undertuning; we
predict the same binding on any deployment with a noisy,
undertuned, or drifting training-time expert.

\subsection{Decision Rule: Which Method to Deploy When}
\label{sec:when-to-use}

The failure modes convert into a deployment-scenario decision rule
(\cref{tab:decision_rule}) keyed on three observables: expert
quality $q^{\mathrm{exp}} \equiv J_{\mathrm{exp}}/J_{\mathrm{RL}}$
(reported as a percentage in the table), task termination, and
expected perturbation type. The denominator is the practitioner's
main calibration cost: a $\sim$$10\%$ pilot SAC run, a published
return on a structurally similar task, or, absent either, the
conservative branch ($q^{\mathrm{exp}}$ unknown $\to$ treat as
low). $J_{\mathrm{RL}}$ here is the no-expert-SAC asymptote, not
ENA's $J_{\mathrm{ref}}$ physical ceiling: on FourTank
$q^{\mathrm{exp}}\!=\!98\%$ even though the expert hits only
$\sim$$49\%$ of $J_{\mathrm{ref}}$. ``RL-near-ceiling'' throughout
refers to the $J_{\mathrm{RL}}$ scale.

\begin{table}[h]
\centering
\footnotesize
\renewcommand{\arraystretch}{1.15}
\setlength{\tabcolsep}{4pt}
\begin{tabular}{@{}llccccccl@{}}
\toprule
\textbf{Scenario} & \textbf{Failure}
& \textbf{IBRL} & \textbf{JSRL-c} & \textbf{JSRL-tt}
& \textbf{Residual} & \textbf{EDGE} & \textbf{Anchor evidence} \\
\midrule
$q^{\mathrm{exp}}\!\gtrsim\!90\%$, non-term., clean   & F1     & $\times$ &          &          & $\checkmark$ &              & FT: IBRL$<$SAC ($p{=}0.031$) \\
$q^{\mathrm{exp}}\!\lesssim\!60\%$, non-term., clean  & F2     & $\checkmark$ & $\checkmark$ &      & $\times$    &              & CR: Residual $-9\%$ \\
Terminating, clean                          & F1+F2  &          & $\checkmark$ &      & $\times$    & $\checkmark$ & P3DC: Residual capped at $-64\%$ \\
\midrule
Undertuned training-time expert             & F3     &          & $\checkmark$ & $\times$ &          & $\checkmark$ & P3DC $\sigma{=}0.5$ undertune \\
Deployment action / obs noise               & ---    &          &          &          & $\checkmark$ &              & P3DC action-bias $\sigma{=}0.25$ \\
\bottomrule
\end{tabular}
\caption{Decision rule. \emph{Clean} = no expert undertuning and no
deployment-time noise; \emph{terminating} = episodes can end early
on failure (e.g.\ the aircraft crashing on Plane3DCircle).
$\checkmark$ recommends, $\times$ flags a known failure on the
anchor evidence, blank is unconstrained.}
\label{tab:decision_rule}
\end{table}

The two perturbation rows split cleanly: training-time undertuning
favors methods that keep falling back to the expert at every state
(JSRL-curriculum, EDGE); deployment-time noise favors Residual's
bounded correction. When both perturbations are expected together,
no method dominates on the regimes we have.

\subsection{Limitations}
\label{sec:limitations}

\paragraph{Scope.}
Four tasks (compute-bound). EDGE requires a queryable expert at
every state (fixed-dataset demos out of scope), dense rewards, and an
ensemble critic with non-degenerate disagreement. All methods get $30$ Optuna trials per task. EDGE has $14$ search
dimensions vs $10$--$11$ for baselines, so the per-dimension trial
budget is coarser for EDGE. This biases the comparison
\emph{against} our method: EDGE's reported numbers are a lower bound
on what budget-matched search would yield, not an artifact of an
inflated trial budget. Full scope, compute, and HPO discussion in
\cref{app:limitations_extra}.

\paragraph{When the expert is enough.}
On FourTank ($q^{\mathrm{exp}}\!=\!98\%$) and GlassFurnace
($80\%$) no expert-using RL method clears the expert at
$1\mathrm{M}$ steps; we cannot rule out that a longer horizon
($5{-}10\mathrm{M}$) would let the within-noise lag pull ahead.
Within our budget the honest call is to deploy the expert
directly. Closing the remaining gap to the physical ceiling on an
already-strong expert is left to future work.

\section{Conclusion}
\label{sec:conclusion}

In this work we ran a harmonized head-to-head comparison of
query-time-expert RL methods on a shared SAC backbone, common HPO,
and a four-task benchmark, with a degradation sweep covering
expert undertuning and deployment-time perturbations.
No single method dominates: each wins on one task-structure
regime and fails predictably elsewhere. F1/F2/F3
bind where the taxonomy predicts (IBRL $<$ SAC on FourTank,
Residual saturates on Plane3DCircle, JSRL-tt collapses under
expert undertuning), and clean rankings reverse under perturbation.
The benchmark, taxonomy, and decision rule
(\cref{sec:when-to-use}) are the primary contribution; EDGE is
offered as evidence that gate form and scoring rule are
individually exploitable design axes, not as a SOTA claim. Several extensions remain natural follow-ups: a hybrid that mixes
RLPD-style offline-plus-online replay with our online-mixing
gate, a stretch from deterministic continuous-control experts to
stochastic and discrete-action ones, and an evaluation of the same
gate-and-scoring decomposition layered on off-policy backbones
other than SAC.

\bibliographystyle{plainnat}
\bibliography{refs}

\appendix

\section{Critic Blind Spots: Extended Discussion and Proof}
\label{app:proofs}

We state and prove the two formal results behind the informal
``critic blind spot'' motivation in \cref{sec:method}.

\begin{remark}[Critic blind spot under argmax selection]
\label{rem:blindspot}
Let $\mathcal{M}_t = \{s : Q_\phi(s,\pi^{\mathrm{exp}}(s)) > \mathbb{E}_{a\sim
\pi_\theta(\cdot|s)} Q_\phi(s,a)\}$. Under IBRL-style argmax, no
fresh transition $(s, a, \cdot, \cdot)$ with $s\in\mathcal{M}_t$ and
$a\neq\pi^{\mathrm{exp}}(s)$ enters the buffer at step $t$. The Bellman regression
is therefore insensitive to $Q_\phi(s, \pi_\theta(s))$ on
$\mathcal{M}_t$ for as long as $\mathcal{M}_t$ persists.
\end{remark}

\begin{lemma}[Uniform coverage under bounded gate gap]
\label{lem:coverage}
Let $\Delta_t(s) := \check{Q}_t(s,\pi^{\mathrm{exp}}(s)) - \check{Q}_t(s,
\pi_\theta(s))$ be the gate gap on visited states, and suppose
$\sup_{s\in\mathrm{supp}(\mu_t)}\Delta_t(s) \le \Delta_{\max}$ at
step $t$. Then the gate \eqref{eq:behavior} with fixed bandwidth
$\tau$ satisfies the \emph{uniform} bound $p_t(s) \le p_{\max} :=
\sigma(\Delta_{\max}/\tau) < 1$, and policy-arm Bellman anchors
accumulate at rate at least $(1 - p_{\max})\mu_t(s)$ at every
visited $s$. Consequently, no analogue of $\mathcal{M}_t$ persists
on the visited support so long as the gap stays bounded.
\end{lemma}

\paragraph{Why the bias is mechanistic.} Whenever the expert wins at
state $s$, IBRL-style argmax stores $(s, \pi^{\mathrm{exp}}(s), r, s')$ in the
replay buffer, and the critic's Bellman regression target on that
tuple anchors $Q_\phi(s, \pi^{\mathrm{exp}}(s))$, not $Q_\phi(s, a)$ for
$a \neq \pi^{\mathrm{exp}}(s)$. The SAC actor update reads
$\nabla_\theta \mathbb{E}_{s \sim \mathcal{B},\, a \sim \pi_\theta(\cdot
| s)}[Q_\phi(s, a) - \alpha \log \pi_\theta(a|s)]$: the stored action
does not enter the gradient directly, but the gradient's quality
depends on $Q_\phi(s, a)$ being calibrated for the actions
$\pi_\theta$ actually proposes. Older transitions written before
$\mathcal{M}_t$ took its current shape may anchor the policy arm,
but that anchor freezes as soon as the gap-sign flips:
$Q_\phi(s, \pi_\theta(s))$ on $\mathcal{M}_t$ is then determined by
stale data and function-approximator generalization, while it is
exactly the quantity the actor improves against. The failure is
mechanistic, not statistical: even with unlimited environment data,
pure argmax records no fresh transition whose Bellman target updates
$Q_\phi$ at policy-proposed actions on $\mathcal{M}_t$.

\begin{proof}[Proof of \cref{lem:coverage}]
Conditional on the state, the gate is a Bernoulli draw with parameter
$p_t(s) = \sigma(\Delta_t(s)/\tau)$. The logistic $\sigma$ is monotone
increasing, so the gap upper bound transports to
$p_t(s) \le \sigma(\Delta_{\max}/\tau) =: p_{\max} < 1$ for every
$s\in\mathrm{supp}(\mu_t)$. The per-step rate of policy-action
Bellman anchors at $s$ is therefore at least $(1 - p_{\max})\mu_t(s)$
by the law of total probability over $\mu_t$, uniformly in $s$.
\qedhere
\end{proof}

\paragraph{Where the bound enters and where it can fail.} The
hypothesis $\Delta_t(s) \le \Delta_{\max}$ is what converts pointwise
positivity ($p_t(s) < 1$ by the logistic) into a uniform rate that
rules out exponentially-vanishing policy-arm coverage. With unbounded
rewards or critic outputs, $\Delta_t$ can grow without bound on a
state where the expert wins by an arbitrary margin and $p_t(s)\to 1$
becomes possible at any rate. In our benchmark $r$ is bounded and the
critic ensemble is regularized, so the hypothesis holds with
$\Delta_{\max}$ on the order of the per-task return scale; an
empirical estimate of $p_{\max} = \max_{s\in\mathcal{B}}p_t(s)$ over
recent buffer samples is the obvious diagnostic and is the quantity
one would log to verify the lemma's hypothesis at training time.

\section{Per-Environment Results}
\label{app:perenv}

This appendix reports per-environment training curves, expert-bias
trajectories, the uncensored per-env headline, and the cross-env
aggregated rliable view as a sanity check against
\cref{fig:headline} and \cref{tab:main_perenv}.

\begin{figure}[h]
\centering
\includegraphics[width=\linewidth]{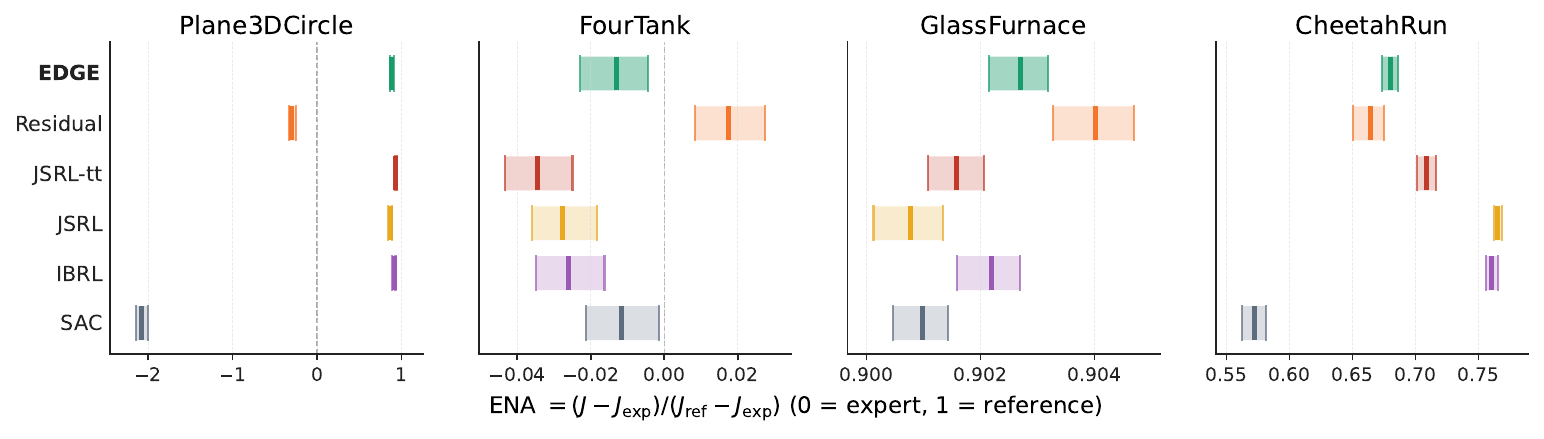}
\caption{Uncensored per-env headline (companion to
\cref{fig:headline}). The Plane3DCircle panel is shown on its
full y-axis: SAC reaches $-6{,}099$ and Residual SAC $3{,}376$, both
of which dwarf the spread among the four top methods when shown
together. The censored body figure trades full disclosure for visual
resolution on the comparison that matters for the F1/F2/F3 story.}
\label{fig:headline_uncensored}
\end{figure}

\begin{figure}[h]
\centering
\includegraphics[width=0.92\linewidth]{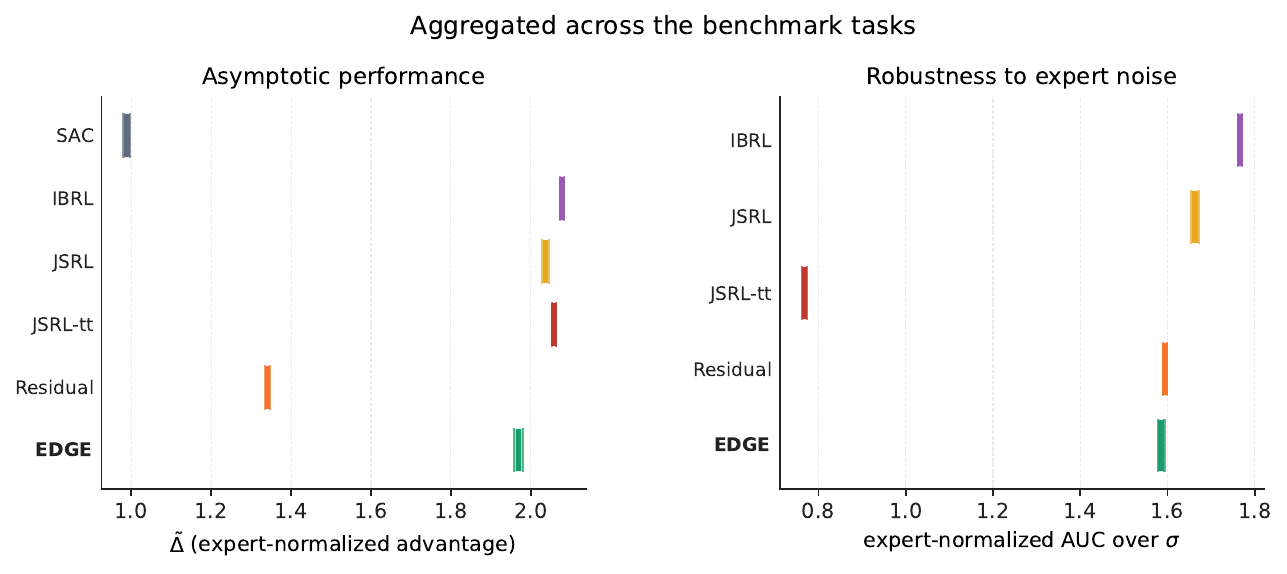}
\caption{Cross-env aggregated rliable view: per-method IQM (tick)
and stratified-bootstrap IQR (rectangle), aggregated across envs.
Right panel: mean $\nbias$ over the undertuning $\sigma$ sweep.
Demoted from the body because the cross-env aggregate hides the
F1/F2/F3 mechanism story visible in \cref{fig:headline}.}
\label{fig:aggregated}
\end{figure}

\begin{figure}[h]
\centering
\includegraphics[width=0.48\linewidth]{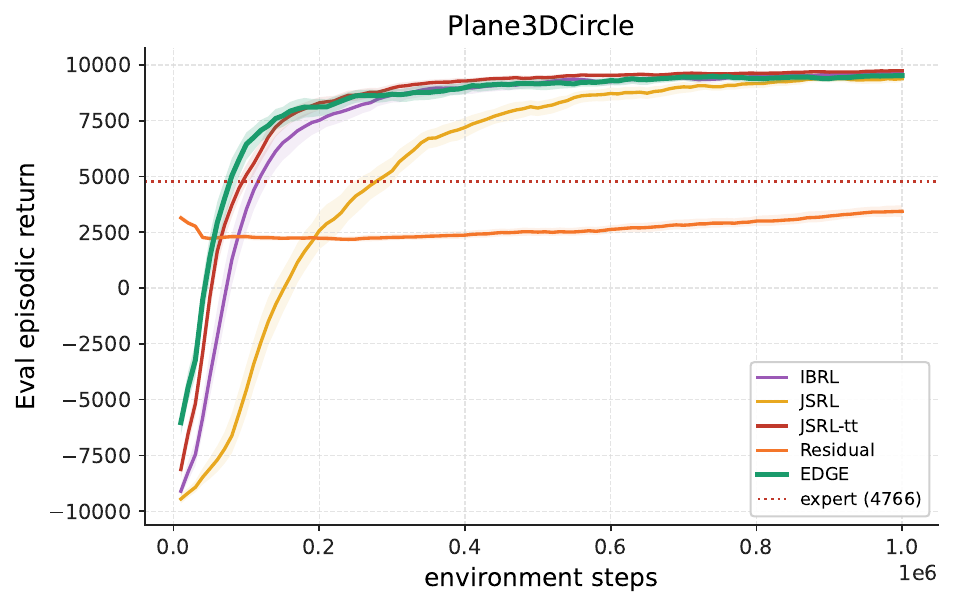}\hfill
\includegraphics[width=0.48\linewidth]{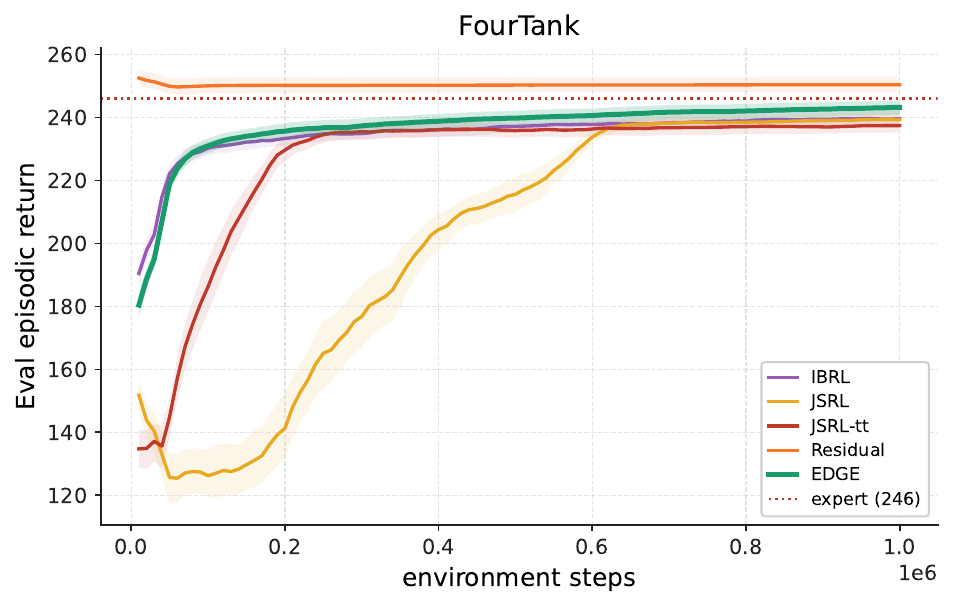}\\
\includegraphics[width=0.48\linewidth]{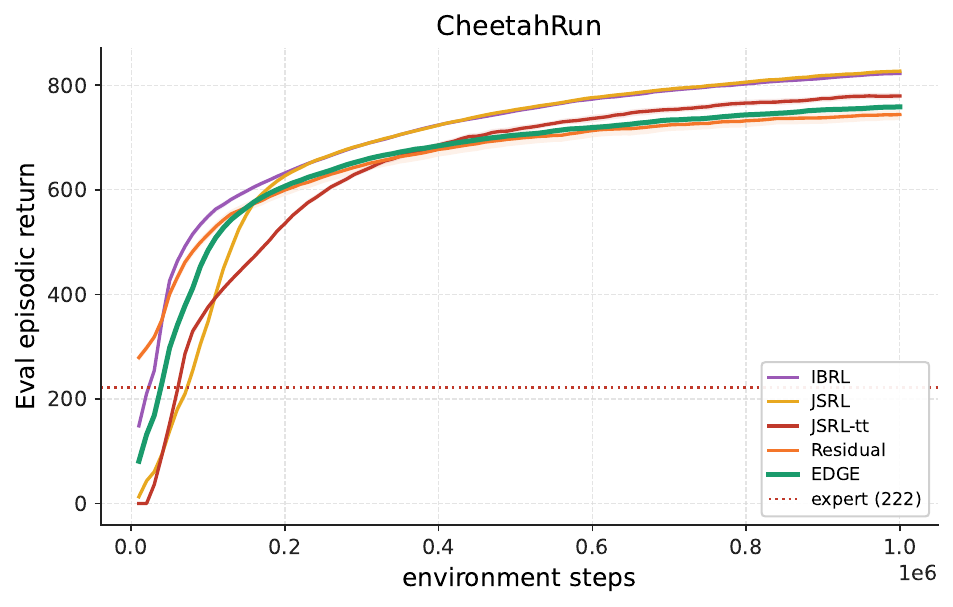}
\caption{Per-environment learning curves. IQM across 100/50 seeds
(control/locomotion) with bootstrap 95\% CI band. Horizontal dashed
line marks $J_{\mathrm{exp}}$.}
\label{fig:curves_appendix}
\end{figure}

\begin{figure}[h]
\centering
\includegraphics[width=0.48\linewidth]{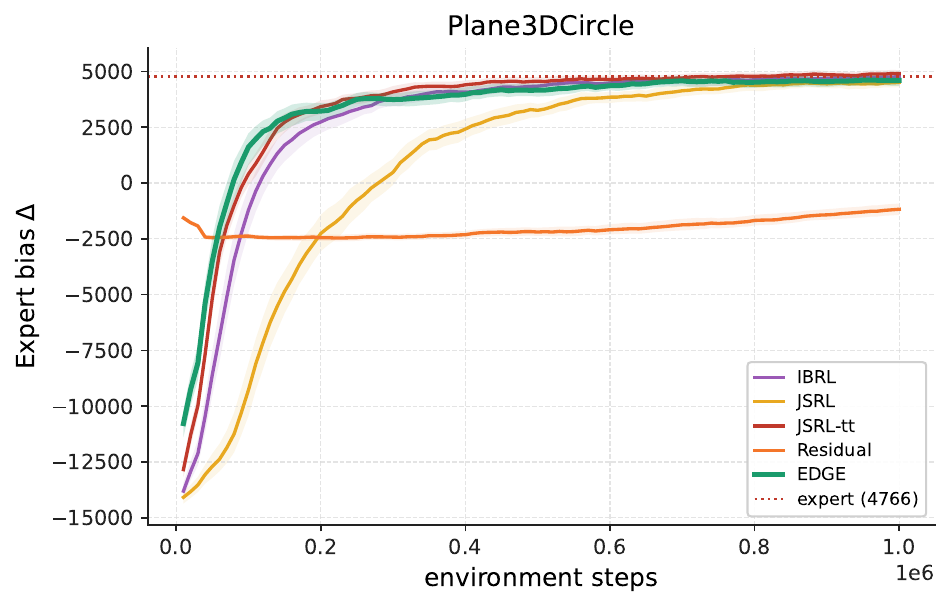}\hfill
\includegraphics[width=0.48\linewidth]{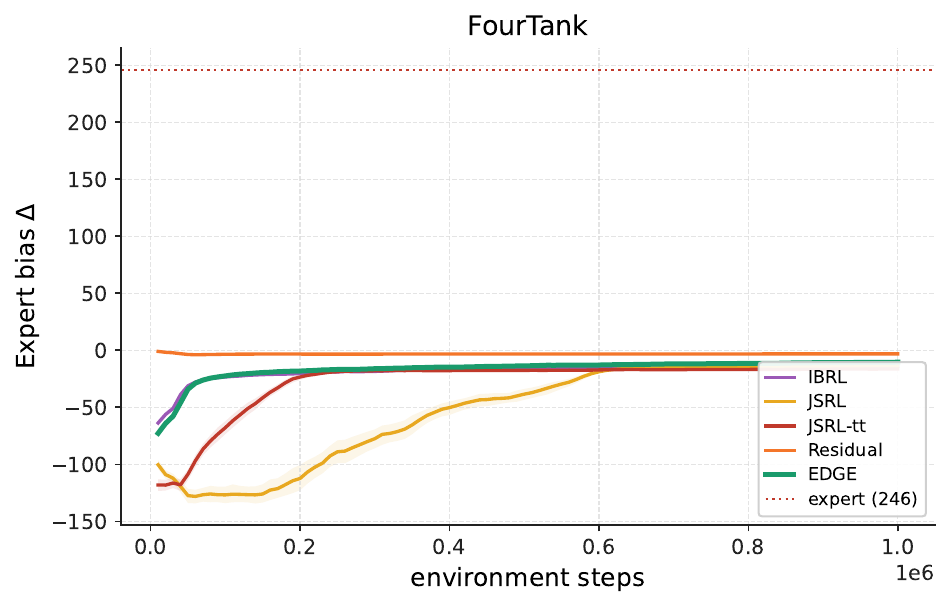}\\
\includegraphics[width=0.48\linewidth]{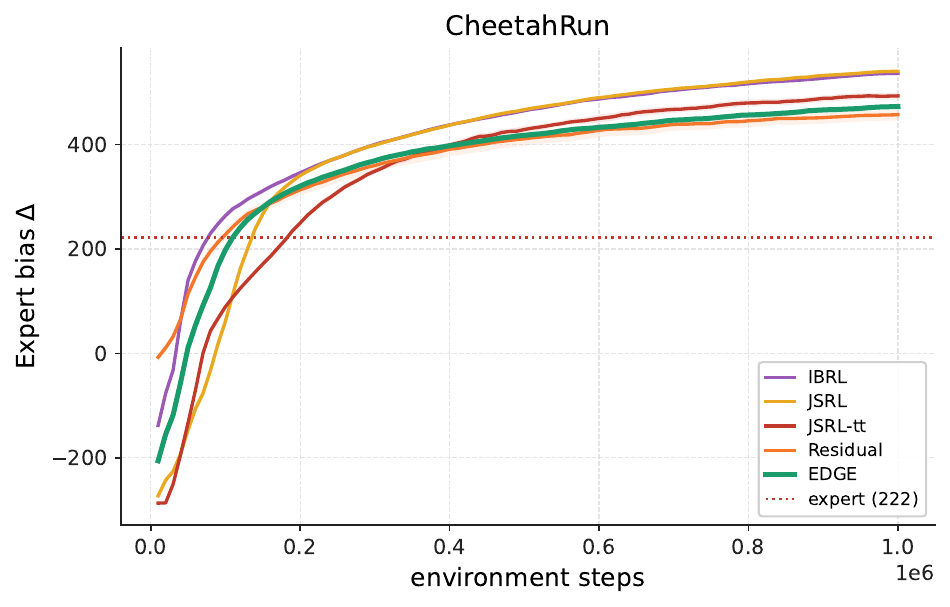}
\caption{Per-environment expert bias $\Delta = J - J_{\mathrm{exp}}$. The
horizontal axis at zero marks expert parity.}
\label{fig:bias_appendix}
\end{figure}

\begin{figure}[h]
\centering
\IfFileExists{figs/sample_eff_steps_per_env.pdf}{
\includegraphics[width=\linewidth]{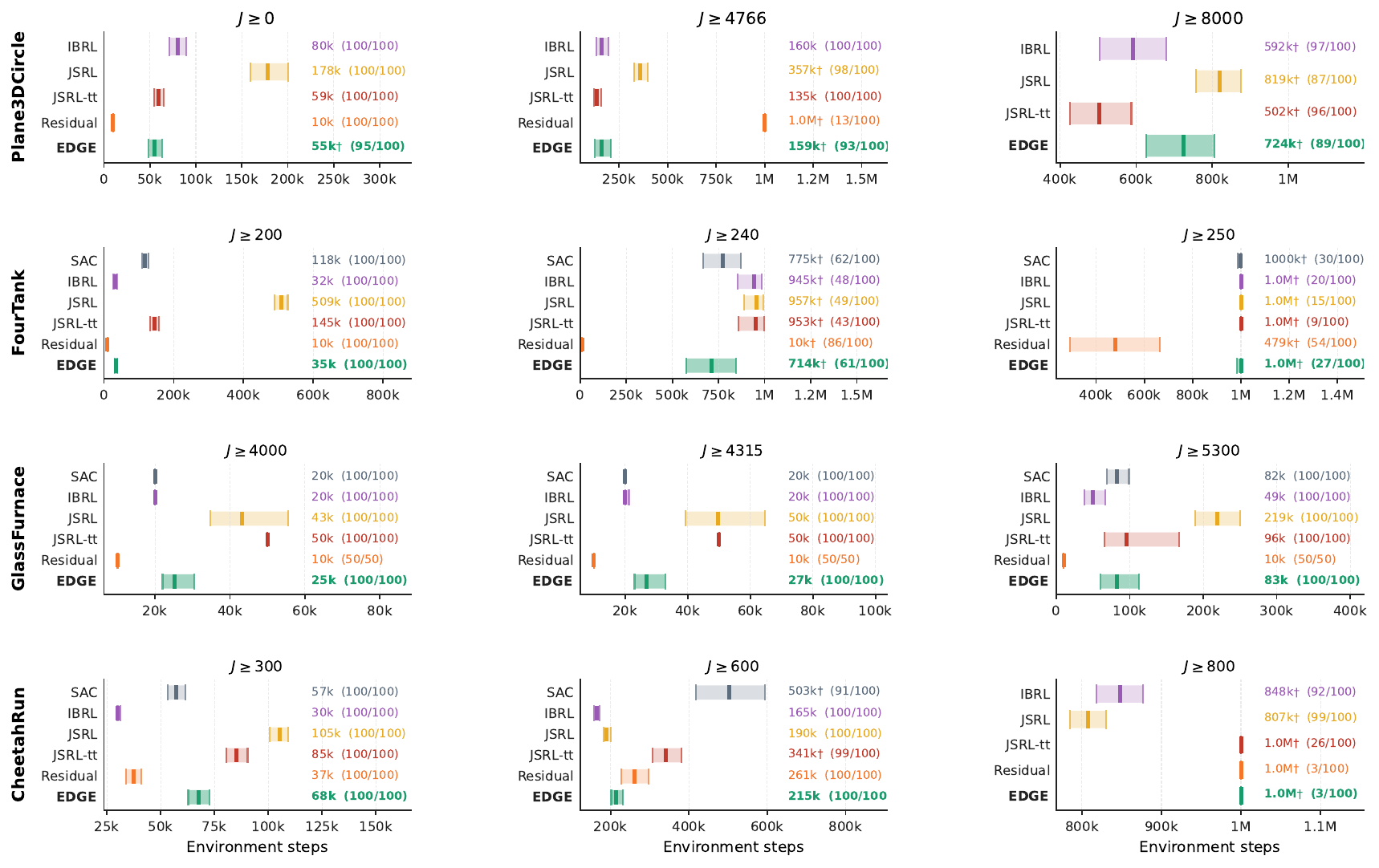}}{}
\caption{Sample efficiency: per-method IQM (tick) and 95\%
bootstrap CI (rectangle) of the first-permanent-crossing step at
three reward thresholds per env. Annotation: \texttt{<IQM>
(<n\_crossed>/<n\_total>)}; $\dagger$ = budget-censored.
\texttt{10k} means the agent already exceeds the threshold at
the first evaluation.}
\label{fig:sample_eff_appendix}
\end{figure}

\begin{figure}[h]
\centering
\IfFileExists{figs/degradation_v2_bars_by_method_grid_with_labels.pdf}{
  \includegraphics[width=\linewidth]{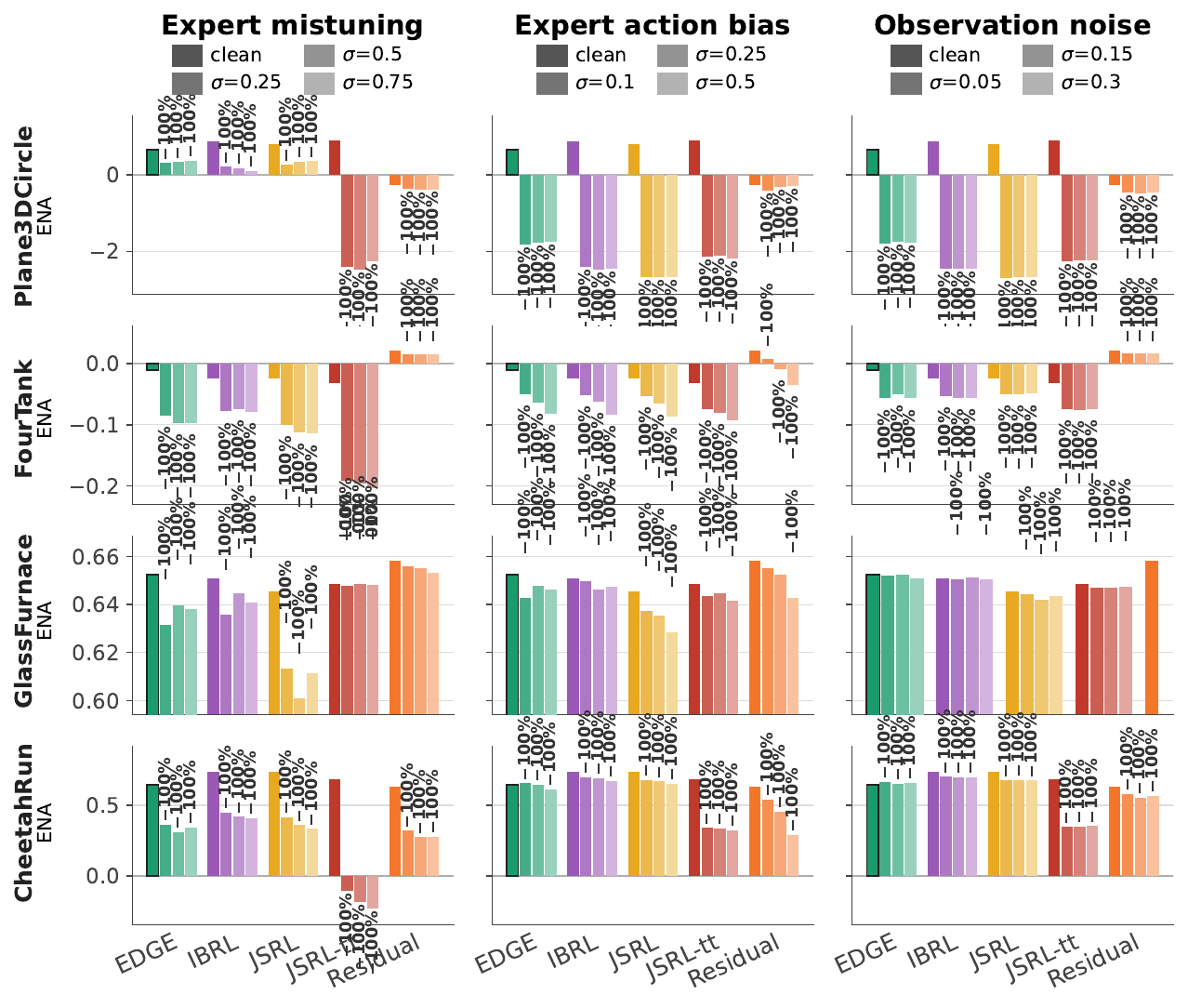}}{}
\caption{Per-environment perturbation sweep, ENA scale with per-bar
percent annotations (annotated companion to
\cref{fig:degradation_body}).}
\label{fig:degradation_appendix}
\end{figure}

\begin{figure}[h]
\centering
\IfFileExists{figs/degradation_v2_bars_by_method_grid_raw.pdf}{
  \includegraphics[width=\linewidth]{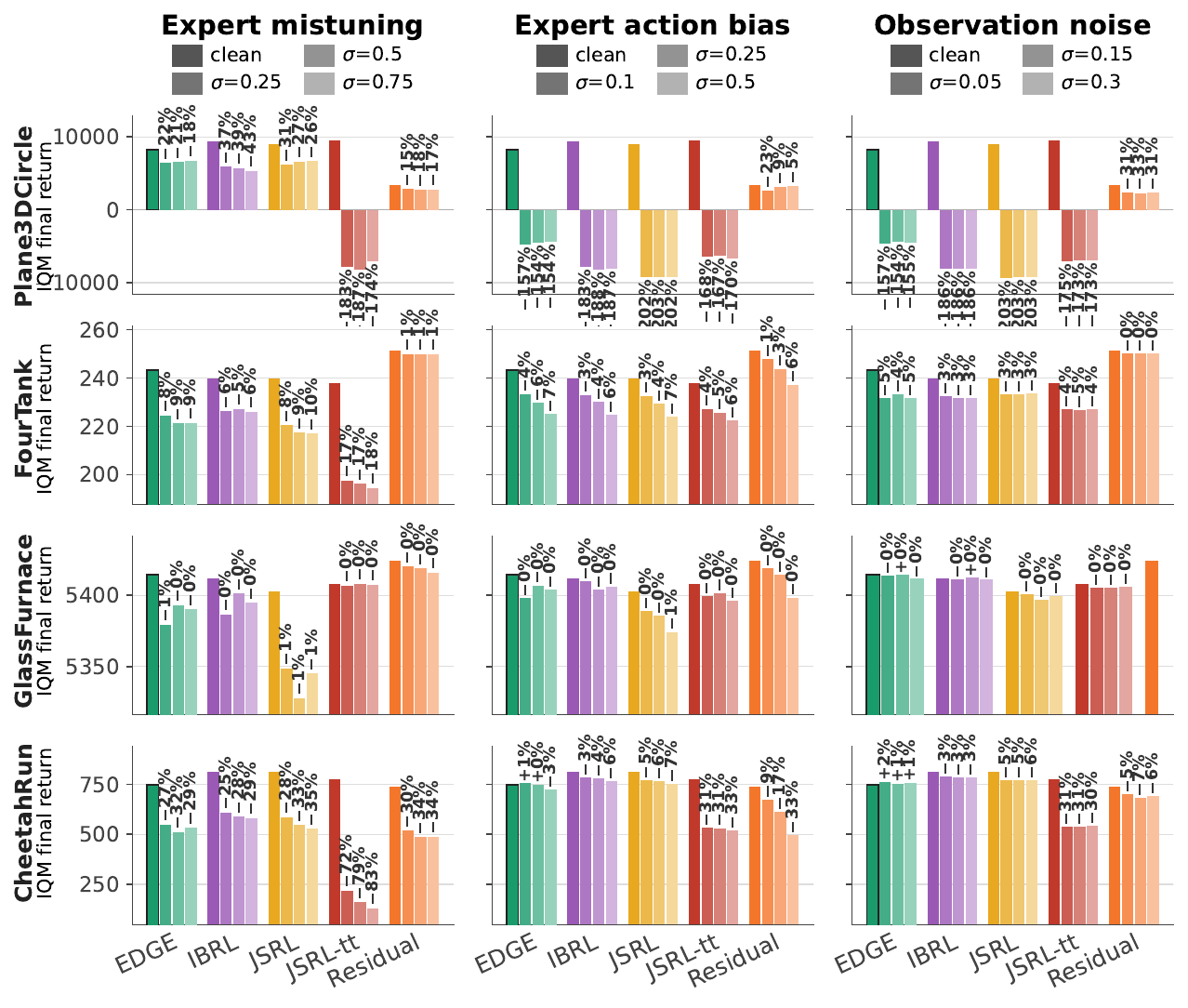}}{}
\caption{Per-environment perturbation sweep, raw IQM final return
(un-normalized companion to \cref{fig:degradation_body}, for direct
inspection in native task units).}
\label{fig:degradation_appendix_raw}
\end{figure}

\section{Ablation Study}
\label{app:ablation}

EDGE composes two design choices the abstract foregrounds (gating
form and pessimism) plus four secondary choices (preprocessing,
critic bootstrap, expert quality sanity, off-policy correctness).
Each variant below flips a single knob of the tuned EDGE config
from the \texttt{sac\_quality\_aware} HPO study and retrains at
$300\mathrm{k}$ steps with $50$ seeds (a shorter horizon than the
$1\mathrm{M}$-step main study, chosen to fit the per-variant compute
budget; relative cost percentages are computed against an
EDGE proxy at the same horizon). Variants run on Plane3DCircle,
CheetahRun, and FourTank: Plane3DCircle has the largest
between-method spread of any benchmark task and gives clean
signal-to-noise for design-choice ablations; CheetahRun stresses
the same knobs on a non-terminating task with a weak expert.
FourTank's per-seed standard deviation ($\sim 10$ on a $\sim 240$
scale) makes sub-method-swap effects borderline at $50$ seeds, but
we run the same $13$ variants there too because FourTank is the
F1-anchor task and reviewer-flagged: see the FT panel
(\cref{fig:ablation_ft}) below.

\paragraph{Variant $\to$ axis map.}\leavevmode\par\nobreak\noindent
\begingroup
\footnotesize
\setlength{\tabcolsep}{4pt}
\renewcommand{\arraystretch}{1.15}
\begin{tabularx}{\linewidth}{@{}l >{\raggedright\arraybackslash}X >{\raggedright\arraybackslash}X@{}}
\toprule
\textbf{Variant} & \textbf{Knob flipped} & \textbf{Axis tested} \\
\midrule
\texttt{gating\_argmax}        & softmax over LCB replaced by argmax over mean-$Q$ & gate form and pessimism, jointly \\
\texttt{bootstrap\_argmax}     & add IBRL-style $\max(Q_\pi, Q_*)$ TD target       & critic-bootstrap composition with IBRL \\
\texttt{bootstrap\_lcb\_gated} & apply LCB gate at the TD bootstrap too            & LCB consistency at value vs.\ action \\
\texttt{no\_state\_aug}        & drop obs $\oplus$ expert-state augmentation       & preprocessing (non-Markov fix) \\
\texttt{no\_obs\_norm}         & drop running obs normalization                    & preprocessing (scale invariance) \\
\texttt{random\_expert}        & PID/CPG expert replaced by stateful random policy & expert-signal sanity \\
\texttt{store\_policy\_action} & buffer stores policy action, not executed         & off-policy correctness \\
\texttt{expert\_prefill}       & $100\mathrm{k}$ expert prefill on top of EDGE     & orthogonal additivity with the F3 fix \\
\texttt{literal\_thompson}     & softmax-LCB gate replaced by literal Gaussian Thompson sampling over both expert and policy streams & gate parameterization (TS vs softmax) \\
\bottomrule
\end{tabularx}
\endgroup

\paragraph{Pessimism and gate-sharpness axes (HPO sweep).}
The two EDGE-specific knobs are tuned per task by HPO over
$\kappa \in [0.1, 5.0]$ and $\tau \in [0.1, 10.0]$ (log-scale). The
tuned values vary by an order of magnitude across tasks
(\cref{tab:beta_per_task}): a $30\times$ spread on $\kappa$
(0.12 to 4.05) and $50\times$ on $\tau$ (0.17 to 8.51) shows both
axes are exploitable. Plane3DCircle (catastrophic failure, weak
expert) lands at strong pessimism + wide gate ($\kappa\!=\!4.05,
\tau\!=\!8.51$); FourTank (no termination, near-ceiling expert) hits
the lower $\kappa$ bound with moderate $\tau$.

\begin{table}[h]
\centering
\small
\begin{tabular}{lcccc}
\toprule
\textbf{Task} & \textbf{Tuned $\kappa$} & \textbf{Tuned $\tau$} & \textbf{Termination?} & \textbf{Expert quality} \\
\midrule
Plane3DCircle    & 4.05 & 8.51 & yes (crash) & weak (50\% of best RL) \\
GlassFurnace     & 2.61 & 0.24 & no          & moderate (80\% of best RL) \\
CheetahRun       & 1.16 & 0.17 & no          & weak (36\% of best RL) \\
FourTank         & 0.12 (lower bound) & 1.64 & no & near-ceiling (98\% of best RL) \\
\bottomrule
\end{tabular}
\caption{HPO-tuned $(\kappa, \tau)$ for EDGE across four envs.}
\label{tab:beta_per_task}
\end{table}

\paragraph{Gate form versus scoring rule: four-corner status.}
The abstract's two-axis claim (gate form and scoring rule) expands
into four corners. The variant names below label each corner by
what the implementation actually computes (argmax/softmax over
min-of-ensemble $Q$ vs.\ LCB score $\check Q = \mu - \kappa\sigma$);
the populated four-corner IQM table is reported alongside the
literal-Thompson discussion below.
\begin{center}
\small
\begin{tabular}{lll}
\toprule
& \textbf{argmax gate} & \textbf{softmax gate} \\
\midrule
\textbf{min-$Q$ scoring}   & \texttt{gating\_argmax}        & \texttt{no\_pessimism} ($\kappa{=}0$) \\
\textbf{LCB scoring}       & \texttt{argmax\_lcb}           & EDGE default (all envs)               \\
\bottomrule
\end{tabular}
\end{center}
All four corners are populated by ablations on Plane3DCircle.

\paragraph{Literal Thompson sampling vs softmax-LCB gate.}
The \texttt{literal\_thompson} variant replaces EDGE's
fixed-bandwidth softmax over LCB scores with literal Gaussian-
parametric Thompson sampling on the same critic ensemble: at each
state the gate draws $\tilde q_e \sim \mathcal{N}(\mu_e, \sigma_e^2)$
and $\tilde q_p \sim \mathcal{N}(\mu_p, \sigma_p^2)$ from the empirical
ensemble mean/std and selects the larger draw. On Plane3DCircle this
underperforms EDGE by $\approx 1{,}965$ IQM ($4{,}265$ vs $6{,}230$);
on CheetahRun it ties EDGE within seed noise ($578$ vs $\approx 590$,
on a task whose ablation spread is small).

A natural reading of this gap is the under-dispersion of small
critic ensembles
\citep{osband2016bootstrappeddqn,an2021edac,ghasemipour2022msg}: with
$N\!\in\!\{2,4\}$, the empirical ensemble std is a biased low
estimator of the posterior std, so the Thompson draws collapse
toward $\mathrm{argmax}_a \mu_a$ and the policy arm gets starved.
The \texttt{literal\_thompson\_k10} variant directly tests this by
re-running literal Thompson sampling with a $N{=}10$ ensemble
(roughly $2.5\times$ the per-step critic cost of $N{=}4$). On
Plane3DCircle the gap to EDGE is essentially unchanged: $4{,}190$ at
$N{=}10$ vs $4{,}265$ at $N{\in}\{2,4\}$, both well below EDGE's
$6{,}230$. We therefore reject under-dispersion as the operative
mechanism for the EDGE-vs-literal-TS gap.

The cleaner reading, consistent with the four-corner table below,
is that the load-bearing piece is the \emph{pessimism shift}
that LCB scoring applies to both arms ($\check Q = \mu - \kappa
\sigma$, with $\kappa$ tuned by HPO), not the gate parameterization.
Literal Thompson sampling is symmetric noise around $\mu$ on each
arm; swapping it for the LCB shift biases the gate against
high-uncertainty actions on both arms, which is what removes the F1
critic-blind-spot pathology.

\paragraph{Gate form and scoring rule, four corners populated.}
With \texttt{argmax\_lcb} closed, all four corners are populated on
Plane3DCircle ($300\mathrm{k}$ steps, $50$ seeds):
\begin{center}
\small
\begin{tabular}{lcc}
\toprule
& \textbf{argmax gate} & \textbf{softmax gate} \\
\midrule
\textbf{min-$Q$ scoring}   & $2{,}293$ & $6{,}109$ \\
\textbf{LCB scoring}       & $5{,}468$ & $\mathbf{6{,}230}$ \\
\bottomrule
\end{tabular}
\end{center}
Reading column-wise (gate-form effect at fixed scoring): softmax
beats argmax by $+3{,}816$ at min-$Q$ and only $+762$ at LCB.
Reading row-wise (scoring effect at fixed gate-form): LCB beats
min-$Q$ by $+3{,}175$ at argmax and $+121$ at softmax. The
$\{+121, +762\}$ off-diagonal increments are within plausible seed
noise at $50$ seeds, so the data more cleanly support \emph{either
lever alone (softmax with min-$Q$, or argmax with LCB) recovers
nearly all of EDGE's score} than a strong "they are not redundant"
claim; EDGE is nominally best at $6{,}230$ but the LCB-vs-min-$Q$
gap at the softmax corner is too small to read as a real
incremental contribution from softmax on top of LCB. The
load-bearing finding is the diagonal: the $(argmax, min-Q)$ corner
that IBRL occupies ($2{,}293$) is by far the worst, and either of
EDGE's two design choices alone is sufficient to escape it. The
literal-TS head-to-head above further isolates LCB's
pessimism-shift as the bigger of the two levers.

\paragraph{Reproducibility.} Per-variant per-env results are
released as
\texttt{ablation\_results/<env>\_\_<variant>.json} in the harness
mirror; the figure and the four-corner table above are regenerated
from those JSONs by \texttt{plot\_sweep.py --paper}.

\begin{figure}[h]
\centering
\IfFileExists{figs/ablation_plane3dcircle.pdf}{
\includegraphics[width=0.85\linewidth]{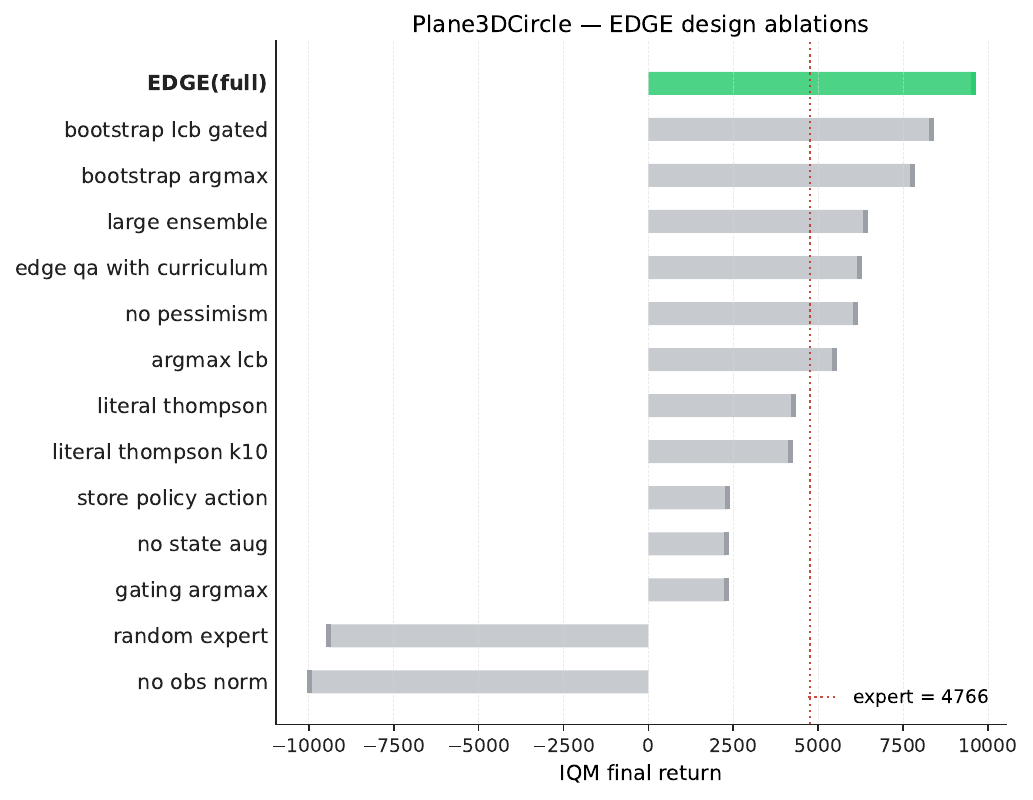}}{}
\caption{Plane3DCircle ablation bars. Reference (green) is EDGE
at the tuned configuration; each grey bar flips a single design
knob. CheetahRun companion in the appendix
(\cref{fig:ablation_cr}); FourTank is not used as an ablation env
(see appendix intro).}
\label{fig:ablation_p3dc}
\end{figure}

\begin{figure}[h]
\centering
\IfFileExists{figs/ablation_cheetahrun.pdf}{
\includegraphics[width=0.85\linewidth]{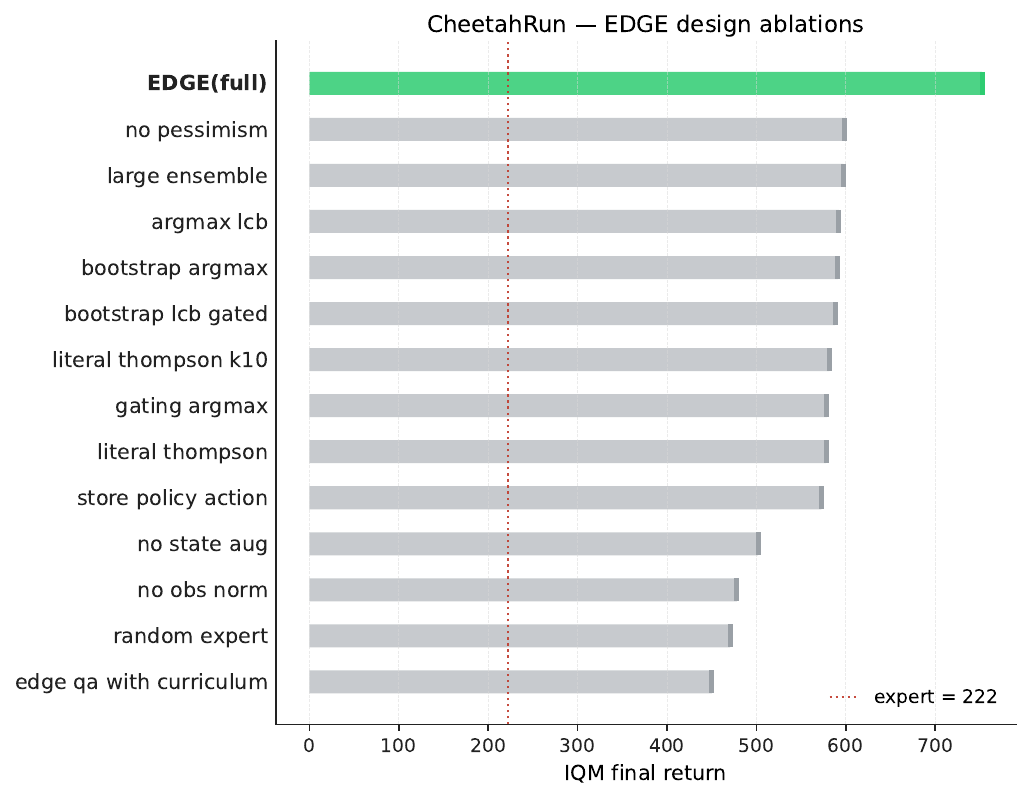}}{}
\caption{CheetahRun ablation bars (companion to
\cref{fig:ablation_p3dc}). Variant spread is small: every variant
sits within $\sim$$25\%$ of EDGE, consistent with CheetahRun's
documented low between-method spread on the main results.}
\label{fig:ablation_cr}
\end{figure}

\begin{figure}[h]
\centering
\IfFileExists{figs/ablation_fourtank.pdf}{
\includegraphics[width=0.85\linewidth]{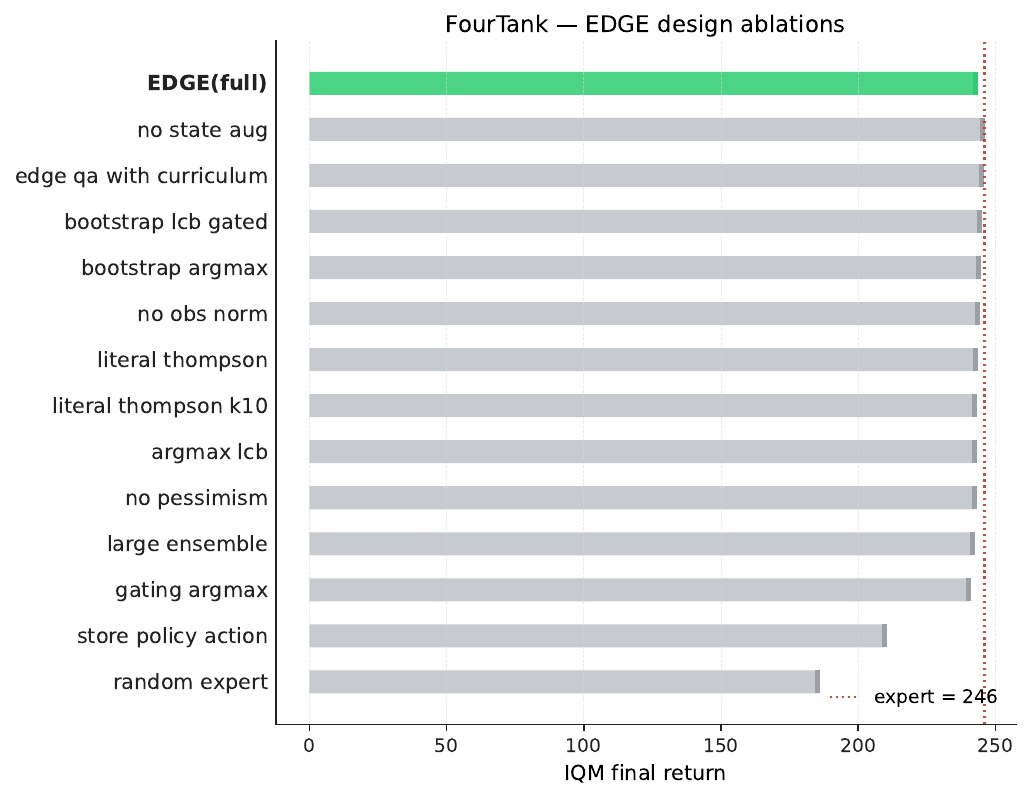}}{}
\caption{FourTank ablation bars (companion to
\cref{fig:ablation_p3dc}). Variant spread is tight ($\sim$$240$ to
$245$ IQM, on a $\sim 250$-scale task) for every variant except
the two sanity-check breakers \texttt{random\_expert} ($185$) and
\texttt{store\_policy\_action} ($210$). Within the tight cluster
all four 2$\times$2 corners (gate form $\times$ scoring rule)
sit within seed noise of EDGE: at the F1-anchor task the
distinguishing axes of EDGE are not individually load-bearing on
the per-seed scale we measure.}
\label{fig:ablation_ft}
\end{figure}

\section{Benchmark Tasks}
\label{app:tasks}

This appendix expands the per-task descriptions summarized in
\cref{sec:baselines}. \cref{tab:expert_headroom} reports
measured expert performance and the per-task reference scale
$J_{\mathrm{ref}}$ used to compute the expert-normalized advantage
$\nbias$ (\cref{eq:nbias}). $J_{\mathrm{ref}}$ is the per-task
episodic-return ceiling (per-step reward upper bound multiplied by
episode length). The ceilings are physical upper bounds, so
$J_{\mathrm{ref}}$ does not depend on any method evaluated in the
main results.

\begin{table}[h]
\centering
\small
\renewcommand{\arraystretch}{1.1}
\begin{tabular}{lrrrrr}
\toprule
\textbf{Task} & \textbf{Episode} & \textbf{Act.\ dim} &
\textbf{Expert} & \textbf{$J_{\mathrm{exp}}/J_{\max}$} & \textbf{cv} \\
\midrule
Plane3DCircle   & 10\,000 & 3  & 4{,}766 & 0.48 & 0.89 \\
FourTank        & 500     & 2  & 246     & 0.49 & 0.18 \\
GlassFurnace    & 5\,760  & 4  & 4{,}315 & 0.75 & 0.43 \\
CheetahRun      & 1\,000  & 6  & 291     & 0.29 & 0.07 \\
\bottomrule
\end{tabular}
\caption{Tuned-expert performance per task (16--20 seed mean).
$J_{\mathrm{exp}}/J_{\max}$ uses the theoretical episodic ceiling
(per-step reward upper bound times episode length: $10^4$ on
Plane3DCircle, $500$ on FourTank, $5{,}760$ on GlassFurnace, $1{,}000$
on CheetahRun). \emph{cv}: coefficient of variation across seeds.}
\label{tab:expert_headroom}
\end{table}

\textbf{Plane3DCircle.} A 3D fixed-wing aircraft (from the anonymized
task suite, \cref{app:reproducibility})
commanded to track a circular trajectory of randomly drawn radius at a target
altitude. Episodes are 10{,}000 steps; the action space is $[-1,1]^3$
(throttle, stick, aileron). The reward multiplies altitude-tracking and
radial-tracking factors, both raised to the 10th power for sharpened
shaping. A fixed-gain heading PID is structurally inadequate because
the reference direction rotates continuously, so the controller leaves
substantial headroom.

\textbf{FourTank.} The quadruple-tank process of
\citet{johansson2000quadruple}, a canonical MIMO benchmark with
right-half-plane zeros, taken from the pc-gym RL benchmark suite
\citep{bloor2024pcgym} and adapted into the anonymized task suite.
Two pumps feed four tanks through crossover valves; the lower two
tank levels must be tracked via the pumps. Episodes are 500 steps.
A per-loop PID is fundamentally MIMO-incorrect (the decoupling
matrix depends on operating point).

\textbf{GlassFurnace.} A custom thermal industrial process from the
anonymized task suite, with multi-zone heating dynamics and slow
thermal coupling. The expert is a mid-quality PID; long horizon
makes the asymptote-vs-degradation tradeoff easier to read than on
shorter setpoint tasks.

\textbf{CheetahRun.} The dm\_control \citep{tassa2018dmcontrol} cheetah
as wrapped by Mujoco Playground \citep{mujocoplayground2024}. Six actuators, 1{,}000-step episodes, reward
is forward velocity plus survival; the task is non-setpoint. The expert is
a Central Pattern Generator (CPG), an open-loop sinusoidal pattern
generator with 13 scalar parameters (one frequency, six amplitudes, six
phases).

\subsection{Visual Renderings}

We render the three non-standard tasks
(\cref{fig:env_renders}). CheetahRun follows the standard
dm\_control cheetah and is not re-rendered.

\begin{figure}[h]
\centering
\begin{minipage}[t]{0.99\linewidth}
\centering
\includegraphics[width=\linewidth]{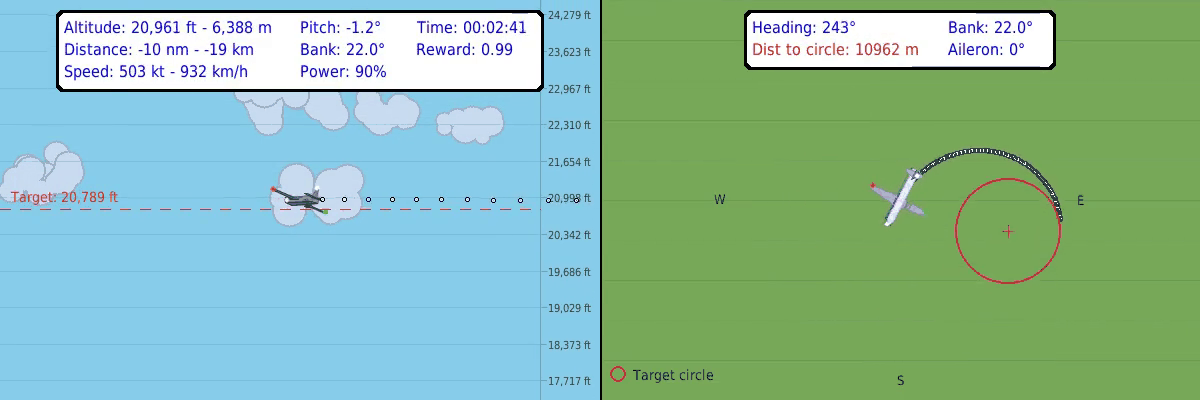}\\[2pt]
\textbf{(a)} Plane3DCircle: aircraft side and top views.
\end{minipage}\\[10pt]
\begin{minipage}[t]{0.49\linewidth}
\centering
\includegraphics[width=\linewidth]{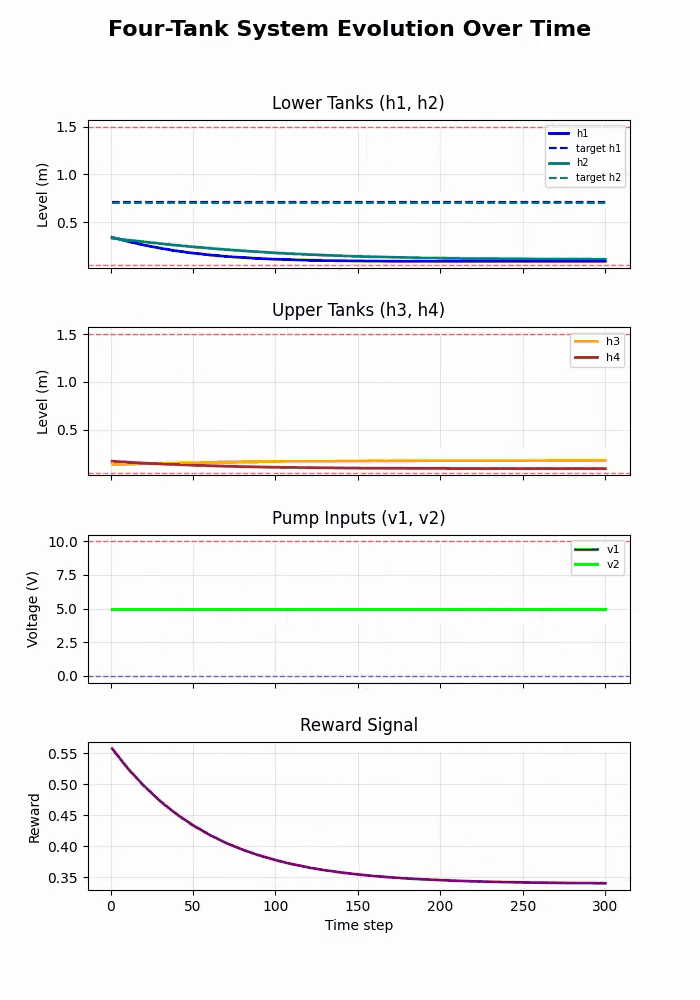}\\[2pt]
\textbf{(b)} FourTank: state evolution under PID.
\end{minipage}\hfill
\begin{minipage}[t]{0.49\linewidth}
\centering
\includegraphics[width=\linewidth]{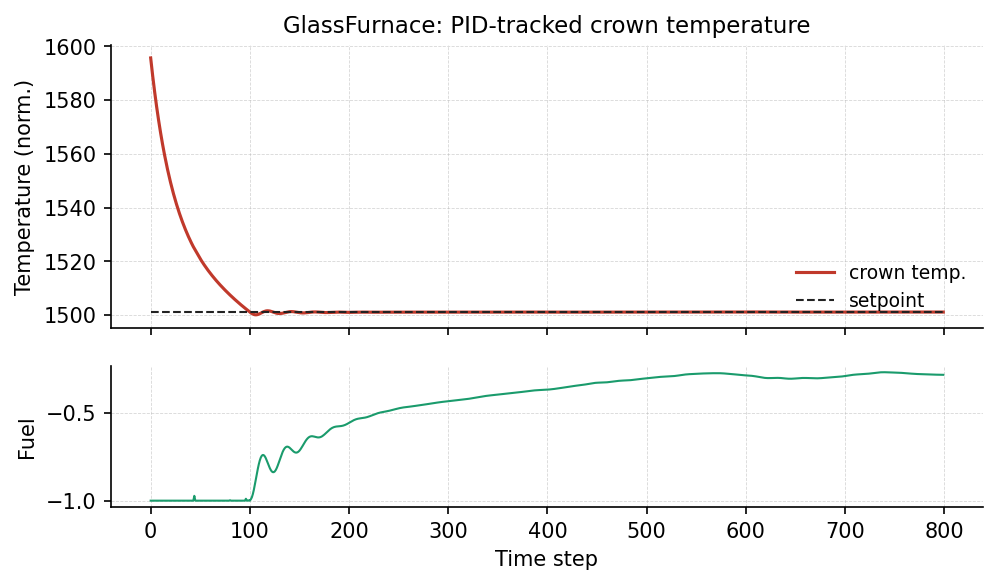}\\[6pt]
\textbf{(c)} GlassFurnace: long thermal time constant makes
setpoint tracking hard despite the simple control loop.
\end{minipage}
\caption{Per-environment renderings for the three non-standard
benchmark tasks.}
\label{fig:env_renders}
\end{figure}

\section{Significance Testing for \cref{tab:main_perenv}}
\label{app:significance}

\paragraph{Per-seed scalar.} For each (env, method) we read the per-seed
final-window mean of \texttt{Eval/episodic\_mean\_reward} over the last
$200{,}000$ env steps from the corresponding TensorBoard event files.
The result is a vector of length $n_{\mathrm{seeds}}$ ($100$ for control
tasks, $50$ for CheetahRun), one scalar per seed.

\paragraph{IQM and dispersion.} The reported center is the
inter-quartile mean: sort the per-seed scalars and average the middle
$50\%$. The brackets are a 95\% bootstrap confidence interval of the
IQM ($5{,}000$ resamples, RNG seed $42$).

\paragraph{$\Delta\%$ vs SAC.} Computed from the IQM values:
$\Delta = (J_{\mathrm{method}} - J_{\mathrm{SAC}})/|J_{\mathrm{SAC}}|
\times 100\%$. Sign convention: positive means the method beats SAC.

\paragraph{Hypothesis test.} For each (env, method $\neq$ SAC) we run a
two-sided Mann-Whitney $U$ test on the per-seed scalars
(\texttt{scipy.stats.mannwhitneyu} \citep{virtanen2020scipy},
``two-sided''). Mann-Whitney is
chosen over a $t$-test because per-seed final-window scores are not
guaranteed normal: P3DCircle and CheetahRun have heavy left tails
(seeds that crash early), and SAC's distribution on P3DCircle is
multi-modal. The rank-based test is order-only and survives both
issues.

\paragraph{Multiple-comparison correction.} Within each env, the five
expert-using methods (IBRL, JSRL-curriculum, JSRL-training-time, Residual, EDGE) form a
test family of size $5$. We apply Holm-Bonferroni step-down correction
to the raw $p$-values within each env. Significance markers reflect
corrected $p$-values: $^{*}{:}p_{\mathrm{corr}}{<}0.05$,
$^{**}{:}p_{\mathrm{corr}}{<}0.01$,
$^{***}{:}p_{\mathrm{corr}}{<}0.001$.

\paragraph{Reproducibility.} The script that regenerates \cref{tab:main_perenv}
is \texttt{make\_main\_table.py} at the repository root. It reads
TensorBoard events, computes the per-cell statistics, and writes
\texttt{paper/tables/main\_perenv.tex}, which the manuscript
\texttt{\textbackslash input}'s. Re-run after new HPO results land
or seed counts change.

\section{Directional Tests for F1, F2, F3}
\label{app:directional_tests}

We confirm each failure mode F1/F2/F3 with one directional test
whose sign is fixed by the F-mechanism's prediction. All three are
one-sided permutation tests on per-seed final-window IQMs ($n{=}50$
per arm, $10^{5}$ shuffles, additive $+1$ smoothing on the count).
Per-seed values are read from the phase3 TB events for F1 and F2,
and from the per-seed dump of \texttt{run\_degradation\_eval.py}
for F3. The regenerating script is
\texttt{run\_directional\_tests.py} at the repository root.

\begin{table}[h]
\centering
\footnotesize
\setlength{\tabcolsep}{4pt}
\renewcommand{\arraystretch}{1.15}
\begin{tabular}{@{}llllrrrl@{}}
\toprule
\textbf{Mode} & \textbf{Setting} & \textbf{H1} & \textbf{IQM(a)} & \textbf{IQM(b)} & \textbf{$\Delta$} & \textbf{$p$} \\
\midrule
F1 & FT clean        & IBRL $<$ SAC        & $239.7$   & $244.0$   & $-4.3$     & $0.031$    \\
F1 & GF clean        & IBRL $<$ SAC        & $5{,}413$ & $5{,}405$ & $+7.9$     & $0.995$    \\
F2 & P3DC clean      & Residual $<$ JSRL-tt & $3{,}237$ & $9{,}638$ & $-6{,}401$ & $<10^{-4}$ \\
F2 & P3DC clean      & Residual $<$ IBRL    & $3{,}237$ & $9{,}499$ & $-6{,}262$ & $<10^{-4}$ \\
F2 & P3DC clean      & Residual $<$ EDGE    & $3{,}237$ & $9{,}318$ & $-6{,}081$ & $<10^{-4}$ \\
\midrule
F3 & P3DC $\sigma{=}0.5$ & JSRL-tt $<$ EDGE & $-6{,}590$ & $-4{,}745$ & $-1{,}845$ & $<10^{-4}$ \\
\bottomrule
\end{tabular}
\caption{One-sided permutation tests on per-seed final-window IQM
($n{=}50$ per arm, $10^{5}$ shuffles, additive smoothing). F1
binds on FourTank ($q^{\mathrm{exp}}\!=\!98\%$) but not on
GlassFurnace ($q^{\mathrm{exp}}\!=\!80\%$): the binding regime is
tighter than ``near-ceiling'' in the loose sense and we narrow the
held-out condition accordingly.}
\label{tab:directional_tests}
\end{table}

\section{Additional Scope and Compute Notes}
\label{app:limitations_extra}

\textbf{Scope.} Evaluation is in simulation with queryable
deterministic experts on continuous-control tasks; we make no
sim-to-real claims. Extension to discrete actions or stochastic
experts (replace $\delta_{\pi^{\mathrm{exp}}(s)}$ in \eqref{eq:behavior} with the
expert action distribution) is straightforward in principle but
untested. Three of five tasks are PID-regulated setpoint-tracking
problems; while they span distinct failure modes of PID control (MIMO
coupling, long-horizon stiffness, 3D kinematic tracking),
high-dimensional manipulation and image-based observations are not
covered by this benchmark.

\textbf{Per-step compute overhead.} Every environment step requires
one expert query, one policy forward pass, and $N$ critic forward
passes at \emph{both} the expert and policy actions to compute
$\check{Q}_t$. This roughly doubles the wall-clock cost of action
selection compared to plain SAC; on environments where simulation
itself is the bottleneck the overhead is amortized, but on
GPU-accelerated batched simulators it is not. Wall-clock costs per
phase, broken down by env and method, are recorded in
\texttt{run\_full\_study.py --status} and reproduced in
\cref{app:hpo}.

\textbf{Total project compute.} All training runs use a single GPU
per run; intra-trial parallelism comes from JAX-vmapped seed batches
(50--100 seeds per training process). We report wall-clock totals
read from per-trial logs (\texttt{elapsed\_s} fields in the result
JSONs).
\begin{itemize}[leftmargin=1.2em,itemsep=1pt,topsep=1pt]
  \item HPO phase 1 (TPE search, 30 trials $\times$ 20 seeds $\times$
        $1\mathrm{M}$ steps, 5 methods $\times$ 4 envs):
        $\sim$$295$ GPU-hours.
  \item HPO phase 2 (top-5 confirmation, 50 seeds $\times$ $1\mathrm{M}$
        steps): $\sim$$258$ GPU-hours.
  \item HPO phase 3 (top-1 confirmation, 50--100 seeds $\times$
        $1\mathrm{M}$ steps): $\sim$$51$ GPU-hours.
  \item Degradation sweeps (v1 expert undertuning $+$ v2 action bias and
        observation noise; 167 cells $\times$ 50 seeds $\times$ $300\mathrm{k}$
        steps): $\sim$$120$ GPU-hours (estimated; per-cell logs do not
        record \texttt{elapsed\_s} for these runs).
  \item Ablation study ($13$ variants $\times$ 2 envs $\times$ $50$
        seeds $\times$ $300\mathrm{k}$ steps): $\sim$$25$ GPU-hours.
\end{itemize}
Total reported compute is approximately $730$ GPU-hours. The full
project budget including failed/preliminary runs not in the paper is
roughly $1.5$--$2\times$ this. Hardware mix: an internal workstation with $4$
$\times$ data-center-class NVIDIA GPUs ($46$\,GB each) for the bulk
of the HPO and ablation work, supplemented by short-lived cloud
instances with $5$ $\times$ NVIDIA RTX 5090 ($32$\,GB) for parallel
degradation sweeps.

\section{Hyperparameter Optimization Protocol}
\label{app:hpo}

All baselines are tuned by a common protocol to ensure a fair
comparison. We use Optuna's TPE sampler with the
\texttt{constant\_liar}/\texttt{multivariate} options enabled and an
asynchronous median pruner. The per-method search budget is 30
trials; each trial trains 20 seeds for $10^6$ environment steps and
reports the interquartile mean of final-window returns. Trials are
pruned once their trailing mean falls below the median of completed
trials at the same training step, with a warmup of 300k steps. The
search space per method mirrors SAC's (learning rates, discount,
target smoothing, entropy target, batch/arch size); EDGE adds two
method-specific dimensions (pessimism coefficient $\kappa$, gate
temperature $\tau$).

\paragraph{Per-task seed counts.} Final evaluation uses 100
freshly-seeded training runs per configuration at the best-found
hyperparameters on the three control tasks (Plane3DCircle,
FourTank, GlassFurnace). On CheetahRun, GPU-RAM constraints bound
the parallel-seed count to 50. The 20-seed-per-trial HPO budget is
identical across all tasks.

\section{Expert Tuning Protocols}
\label{app:tuning}

The main body treats each task's expert as a black-box oracle. This appendix
describes how each expert is actually tuned, why we chose these protocols,
and what tuning would have been unfair.

\paragraph{Scope of the expert claims.}
We are not specialists in PID auto-tuning, gain-scheduling, or central
pattern generator design, and the controllers we ship were tuned with
modest budgets using off-the-shelf protocols described below. They are
included as \emph{representative} hand-engineered controllers of the
kind a practitioner might already have running on a plant, not as
strong-baseline references for those subfields. A domain expert with a
larger tuning budget would likely produce a controller above ours on
every task. The argument of the paper is therefore not "EDGE beats the
best possible PID/CPG"; it is "EDGE turns a plausible
hand-engineered controller, with the same access the practitioner has
to it, into an RL policy that exceeds the controller it was given." The
ablations in \cref{sec:sensitivity} test the same construction
against deliberately weakened experts so the reader can interpolate
between the experts we ship and stronger ones a domain expert could
plausibly produce.

\subsection{PID Experts, Relay Autotuning with Operating-Point Sweep}

We follow the Åström--Hägglund relay-feedback procedure
\citep{astrom1984relay}: inject a bang-bang relay on the actuator, measure
the sustained-oscillation ultimate gain $K_u$ and period $T_u$, and map to
PID gains via the AMIGO tuning rules \citep{astrom2005amigo}. For each
tracked loop we sweep $N=8$ operating points uniformly across the setpoint
range and retain the middle-point gains, mirroring gain-scheduled industrial
practice. MIMO altitude loops are tuned as two sequential SISO loops
(pump-to-altitude and stick-to-altitude), each relay-tuned against the same
altitude error, following the anonymized task suite's quadruple-tank
convention. Cascade loops (aircraft bank-inside-heading) are tuned
inside-out: inner bank loop first, then outer heading loop with bank closed.
All gains are stored in a versioned JSON file to support reproduction and
per-seed ablation.

\subsection{CPG Experts, Differential Evolution on Episodic Return}

CPG parameters are tuned by differential evolution \citep{storn1997de} on
the mean per-seed episodic return, using SciPy's \citep{virtanen2020scipy}
\texttt{differential\_evolution} with population multiplier $5$, $30$
iterations, and 8-seed evaluation per candidate. Search bounds: frequency
$\in [0.5, 5.0]$ Hz, amplitudes $\in [0, 1]$ per actuator, phases $\in
[0, 2\pi]$ per actuator. On CheetahRun this protocol improves the CPG's mean
return from $91.5$ (hand-tuned) to $221.8$ (DE-tuned), a $2.4\times$
improvement at fixed structure. The CheetahRun DE search converged by
iteration $\approx 22$ of 30 in all replicates.

\subsection{Why These Protocols}

Expert tuning is a hidden confound in expert-guided RL comparisons. A
poorly-tuned expert inflates the headroom available to the RL agent, and
an over-tuned expert (one using reward information that the baseline should
not have) would leak the RL problem into the baseline. Relay autotuning
uses only physical plant identification (no reward access); DE-tuning uses
return but at a fixed structural form (open-loop sinusoid) that cannot
trivially exceed the RL upper bound. Both protocols are reproducible from
the plant simulator alone.

\section{Reproducibility and Released Code}
\label{app:reproducibility}

To support full reproducibility we release three GitHub
repositories. The split is functional: one task suite, one off-policy
RL framework, one experiment harness on top of both. All three are
public.

\begin{itemize}
    \item \textbf{Continuous-control task suite} (Plane3DCircle,
    FourTank, GlassFurnace dynamics, reward formulas, expert
    interfaces): \url{https://github.com/YannBerthelot/TargetGym}.
    Self-contained dynamics and reward specifications also appear in
    \cref{app:tasks}, so the description in this paper is
    sufficient to reimplement each task without consulting the
    repository.
    \item \textbf{Off-policy RL framework} (SAC base, expert-mixing
    pipelines, JSRL curriculum and training-time variants, IBRL,
    Residual SAC, EDGE): \url{https://github.com/YannBerthelot/Ajax}.
    \item \textbf{Experiment harness} (HPO, sweeps, evaluation,
    figure-generating scripts):
    \url{https://github.com/YannBerthelot/AjaxExperiments}.
\end{itemize}

Hyperparameters per method per task are stored as JSON files inside
the harness repository, and \cref{app:hpo} cross-references the
relevant paths.

\paragraph{Third-party assets and licenses.}
The benchmark builds on the following external assets, each cited at
first use in the body and re-listed here with version and license
information:

\begin{itemize}[leftmargin=1.2em,itemsep=1pt,topsep=2pt]
  \item \textbf{Soft Actor-Critic} \citep{haarnoja2018sac}: algorithm
        reimplemented from scratch on top of the released RL framework;
        the framework itself is MIT-licensed.
  \item \textbf{IBRL} \citep{hu2023ibrl}, \textbf{JSRL}
        \citep{uchendu2023jsrl}, \textbf{Residual SAC}
        \citep{johannink2019residual}: methods reimplemented in our
        common-protocol framework from the published descriptions.
  \item \textbf{Optuna} \citep{akiba2019optuna} (HPO): MIT license.
  \item \textbf{rliable} \citep{agarwal2021rliable} (aggregation
        statistics): Apache 2.0 license.
  \item \textbf{pc-gym} \citep{bloor2024pcgym} (FourTank dynamics):
        MIT license. We adapt the environment into our JAX-pure
        anonymized task suite.
  \item \textbf{dm\_control} \citep{tassa2018dmcontrol} and
        \textbf{Mujoco Playground} \citep{mujocoplayground2024}
        (CheetahRun): Apache 2.0 license.
  \item \textbf{JAX} \citep{bradbury2018jax}: Apache 2.0 license.
  \item \textbf{SciPy} \citep{virtanen2020scipy}: BSD-3 license (used
        for \texttt{scipy.stats.mannwhitneyu} in the per-cell test
        family).
\end{itemize}
All assets are used in accordance with their licenses, and the
corresponding citations appear in the bibliography.

\newpage
\section*{NeurIPS Paper Checklist}

\begin{enumerate}

\item {\bf Claims}
    \item[] Question: Do the main claims made in the abstract and introduction accurately reflect the paper's contributions and scope?
    \item[] Answer: \answerYes{}
    \item[] Justification: The abstract and \cref{sec:intro} list four contributions (common-protocol benchmark, F1/F2/F3 failure-mode taxonomy, decision rule, EDGE) that map one-to-one to \cref{sec:main-results}, \cref{sec:failuremodes}, \cref{sec:when-to-use}, and \cref{sec:method} respectively; quantitative claims (e.g.\ Residual gap scaling with expert quality, JSRL-training-time fragility) are anchored to specific tables and figures in those sections.

\item {\bf Limitations}
    \item[] Question: Does the paper discuss the limitations of the work performed by the authors?
    \item[] Answer: \answerYes{}
    \item[] Justification: \cref{sec:limitations} dedicates three paragraphs to environment scope (four tasks, no high-dimensional manipulation or image observations), method scope (queryable expert, dense rewards, ensemble critic), and the equal-trial-budget HPO caveat; \cref{app:limitations_extra} expands the latter and details per-step compute overhead.

\item {\bf Theory assumptions and proofs}
    \item[] Question: For each theoretical result, does the paper provide the full set of assumptions and a complete (and correct) proof?
    \item[] Answer: \answerYes{}
    \item[] Justification: The only formal theoretical result is the critic-blind-spot remark (\cref{rem:blindspot}) and the associated Coverage Lemma in \cref{app:proofs}; assumptions (argmax-style action selection, persistent expert dominance) are stated in the remark, and the appendix gives the proof.

    \item {\bf Experimental result reproducibility}
    \item[] Question: Does the paper fully disclose all the information needed to reproduce the main experimental results of the paper to the extent that it affects the main claims and/or conclusions of the paper (regardless of whether the code and data are provided or not)?
    \item[] Answer: \answerYes{}
    \item[] Justification: \cref{sec:baselines}, \cref{sec:protocol}, and \cref{app:hpo} document the HPO protocol (Optuna TPE, 30 trials, three-phase confirmation), search spaces, seed counts (100/50), and evaluation windows; \cref{app:tasks} specifies dynamics, observation/action spaces, and reward forms per task; \cref{app:significance} documents the test methodology; \cref{app:reproducibility} releases anonymized code mirrors covering the task suite, the RL framework, and the experiment harness.

\item {\bf Open access to data and code}
    \item[] Question: Does the paper provide open access to the data and code, with sufficient instructions to faithfully reproduce the main experimental results, as described in supplemental material?
    \item[] Answer: \answerYes{}
    \item[] Justification: Three public GitHub repositories (\cref{app:reproducibility}) provide the continuous-control task suite (\url{https://github.com/YannBerthelot/TargetGym}), the off-policy RL framework with all five expert-guided methods (\url{https://github.com/YannBerthelot/Ajax}), and the experiment harness covering HPO, sweeps, evaluation, and figure-generating scripts (\url{https://github.com/YannBerthelot/AjaxExperiments}).

\item {\bf Experimental setting/details}
    \item[] Question: Does the paper specify all the training and test details (e.g., data splits, hyperparameters, how they were chosen, type of optimizer) necessary to understand the results?
    \item[] Answer: \answerYes{}
    \item[] Justification: \cref{sec:baselines} lists the SAC backbone shared across all methods; \cref{sec:protocol} fixes the 1M-step training horizon, 200k-step trailing window, IQM aggregation, and 100/50 seed counts; \cref{app:hpo} specifies the Optuna TPE search, per-method search spaces, three-phase pipeline, and pruning rule; per-task tuned hyperparameters for EDGE are in \cref{tab:beta_per_task} and the released mirror.

\item {\bf Experiment statistical significance}
    \item[] Question: Does the paper report error bars suitably and correctly defined or other appropriate information about the statistical significance of the experiments?
    \item[] Answer: \answerYes{}
    \item[] Justification: \cref{tab:main_perenv} reports 95\% bootstrap confidence intervals of the IQM (factor of variability: per-seed final-window return) computed via the rliable \citep{agarwal2021rliable} stratified bootstrap, with $\Delta\%$-vs-SAC significance markers from Mann-Whitney U with Holm-Bonferroni correction at $\alpha{=}0.05$ (methodology in \cref{app:significance}); a pre-registered one-sided permutation test of F1 (IBRL $<$ SAC on FourTank, $n{=}50$ per arm, $10^5$ shuffles) is reported in \cref{sec:protocol}.

\item {\bf Experiments compute resources}
    \item[] Question: For each experiment, does the paper provide sufficient information on the computer resources (type of compute workers, memory, time of execution) needed to reproduce the experiments?
    \item[] Answer: \answerYes{}
    \item[] Justification: \cref{app:limitations_extra} (\textbf{Total project compute}) gives a per-phase breakdown ($\sim$295 GPU-h phase 1, $\sim$258 GPU-h phase 2, $\sim$51 GPU-h phase 3, $\sim$120 GPU-h degradation sweeps, $\sim$8 GPU-h ablations; $\sim$730 GPU-h total reported, $\sim$$1.5$--$2\times$ including failed/preliminary runs) and the hardware mix (4 data-center-class GPUs at 46\,GB plus short-lived 5\,$\times$\,RTX 5090 cloud instances).

\item {\bf Code of ethics}
    \item[] Question: Does the research conducted in the paper conform, in every respect, with the NeurIPS Code of Ethics \url{https://neurips.cc/public/EthicsGuidelines}?
    \item[] Answer: \answerYes{}
    \item[] Justification: The work is foundational benchmarking on simulation tasks: no human subjects, no scraped data, no high-misuse-risk asset released; double-blind anonymity is preserved in both the manuscript and the released code mirrors.

\item {\bf Broader impacts}
    \item[] Question: Does the paper discuss both potential positive societal impacts and negative societal impacts of the work performed?
    \item[] Answer: \answerYes{}
    \item[] Justification: \cref{sec:limitations} (\textbf{Broader impacts} paragraph) discusses positive impacts (reduced per-method tuning effort for practitioners deploying RL on top of existing controllers; the F3 robustness-vs-asymptote tradeoff itself surfaces a deployment-time risk) and negative impacts (methods that beat tuned PIDs may be deployed without proper safety verification; training-time-handoff brittleness can mask deployment failures).

\item {\bf Safeguards}
    \item[] Question: Does the paper describe safeguards that have been put in place for responsible release of data or models that have a high risk for misuse (e.g., pre-trained language models, image generators, or scraped datasets)?
    \item[] Answer: \answerNA{}
    \item[] Justification: The released artifacts (continuous-control task suite, RL framework, benchmarking harness) are research code for industrial-control benchmarking and pose no high-misuse-risk; no pre-trained generative model, no scraped dataset, no image/text-generation capability is released.

\item {\bf Licenses for existing assets}
    \item[] Question: Are the creators or original owners of assets (e.g., code, data, models), used in the paper, properly credited and are the license and terms of use explicitly mentioned and properly respected?
    \item[] Answer: \answerYes{}
    \item[] Justification: \cref{app:reproducibility} (\textbf{Third-party assets and licenses}) enumerates each external asset with its license: pc-gym (MIT, FourTank dynamics), dm\_control / Mujoco Playground (Apache 2.0, CheetahRun), Optuna (MIT, HPO), rliable (Apache 2.0, aggregation), JAX (Apache 2.0). Each is cited in the bibliography at first use.

\item {\bf New assets}
    \item[] Question: Are new assets introduced in the paper well documented and is the documentation provided alongside the assets?
    \item[] Answer: \answerYes{}
    \item[] Justification: Three new assets are released as public GitHub repositories (\cref{app:reproducibility}): a continuous-control task suite (Plane3DCircle, FourTank, GlassFurnace), an off-policy RL framework with all five expert-guided methods, and the experiment harness; per-task dynamics, reward forms, and observation/action spaces are documented in \cref{app:tasks}, hyperparameter configurations are stored as JSON inside the harness repository, and each repository contains a README with setup and reproduction instructions.

\item {\bf Crowdsourcing and research with human subjects}
    \item[] Question: For crowdsourcing experiments and research with human subjects, does the paper include the full text of instructions given to participants and screenshots, if applicable, as well as details about compensation (if any)?
    \item[] Answer: \answerNA{}
    \item[] Justification: The work involves no crowdsourcing and no human subjects; all experiments are simulation-only on continuous-control tasks.

\item {\bf Institutional review board (IRB) approvals or equivalent for research with human subjects}
    \item[] Question: Does the paper describe potential risks incurred by study participants, whether such risks were disclosed to the subjects, and whether Institutional Review Board (IRB) approvals (or an equivalent approval/review based on the requirements of your country or institution) were obtained?
    \item[] Answer: \answerNA{}
    \item[] Justification: The work involves no human subjects, so IRB approval is not applicable.

\item {\bf Declaration of LLM usage}
    \item[] Question: Does the paper describe the usage of LLMs if it is an important, original, or non-standard component of the core methods in this research? Note that if the LLM is used only for writing, editing, or formatting purposes and does \emph{not} impact the core methodology, scientific rigor, or originality of the research, declaration is not required.

    \item[] Answer: \answerNA{}
    \item[] Justification: LLMs are not part of the core methodology, experiments, or analysis; the EDGE algorithm, the F1/F2/F3 taxonomy, the decision rule, all training pipelines, and all empirical results are produced without LLMs in the loop. For transparency, LLM assistance was used during paper preparation for prose editing, LaTeX cleanup, and basic code-assistance tasks (e.g., refactoring plotting boilerplate, drafting result-aggregation scripts), uses that the NeurIPS LLM policy explicitly exempts from required declaration. All scientific content (claims, numbers, citations, proofs) was verified against the underlying code and data by the authors.

\end{enumerate}

\end{document}